\documentclass{article}
\usepackage{PRIMEarxiv}
\usepackage{iftex}
\ifPDFTeX\usepackage[utf8]{inputenc}\fi
\usepackage[numbers]{natbib}
\usepackage[T1]{fontenc}
\usepackage{mathptmx,microtype,amsmath,amssymb,amsthm,booktabs,array,graphicx}
\usepackage{tikz}
\usetikzlibrary{arrows.meta,positioning,calc}
\usepackage{xcolor}
\definecolor{navy}{HTML}{17466C}
\definecolor{teal}{HTML}{178A85}
\definecolor{orange}{HTML}{C97324}
\usepackage[colorlinks=true,allcolors=navy]{hyperref}
\graphicspath{{figures/}}
\usepackage[font=small,labelfont=bf]{caption}
\usepackage{enumitem,float}
\usepackage{placeins}
\newfloat{algorithm}{tbp}{loa}
\floatname{algorithm}{Algorithm}
\setlist{nosep,leftmargin=*}
\newcommand{\FIL}{\textsc{Finite Interaction Lens}}

\newcommand{\B}{\mathcal{B}}
\newcommand{\Dmix}{\Delta_{AB}}
\newcommand{\norm}[1]{\left\lVert#1\right\rVert}
\newcommand{\E}{\mathcal{E}}
\newtheorem{proposition}{Proposition}
\newtheorem{corollary}{Corollary}
\title{Finite Interaction Lens: Structured Precision\\Sharing for Vision--Language Models}
\author{Nozomu Fujisawa\\[3pt]
Keio University\\[3pt]
\href{mailto:teruyo.sui4@gmail.com}{\texttt{teruyo.sui4@gmail.com}}}
\date{}
\hypersetup{
 pdftitle={Finite Interaction Lens: Structured Precision Sharing for Vision-Language Models},
 pdfauthor={Nozomu Fujisawa},
 pdfsubject={Finite Interaction Lens: structured precision sharing for finite internal interventions},
 pdfkeywords={vision-language models, activation patching, mixed precision, finite interactions}
}
\begin{document}
\maketitle
\begin{abstract}
Identifying which visual interventions interact in a vision--language model requires comparing their joint and separate effects. The resulting four-output difference can be small enough for rounding error to obscure it. We introduce the \FIL{} (FIL), a mixed-precision method for evaluating finite intervention families. At each selected linear projection, FIL decomposes each candidate's input into a native state, two single-intervention components, and a full-dimensional mixed residual. It projects the single-intervention components at higher precision and reuses the same numerical contributions in the joint candidates, while retaining every mixed residual and nonlinear trajectory. We derive a local cancellation identity and a conditional bound on error propagation through a smooth nonlinearity. In one MLP of Qwen3-VL-4B-Instruct, FIL reduces final interaction MSE by 94.4\% on 16 synthetic families and by 99.85\% and 99.35\% in two separate four-image COCO studies, relative to a control that has the same nominal BF16-product budget and higher-precision single interventions but projects each joint candidate independently. One image accounts for most of the absolute gain in each COCO study. Higher-accuracy multi-product BF16 arithmetic on all candidates gives lower aggregate error at similar total query times. Single-product BF16 already recovers the reference choices in both candidate-selection tasks we tested. On these development inputs, FIL improves interaction fidelity at a fixed nominal precision budget. Its practical value depends on a task that needs the extra fidelity.
\end{abstract}

\section{Introduction}
Activation patching probes how a model's output changes when an internal representation is replaced or perturbed. For a VLM, comparing interventions on two image regions shows whether their joint effect differs from the sum of their individual effects. Interaction terms can alter the interpretation of patching results~\cite{vaidyanathan2026}. For image regions, such effects could help identify pairs that deserve closer examination, which makes their accurate measurement worthwhile. The calculation subtracts four outputs, so rounding errors that are small at each endpoint can obscure the interaction~\cite{baligar2026}.

We study the arithmetic used to evaluate a fixed grid of intervention strengths while retaining the model's nonlinear computation. Gradient-based estimators screen large candidate spaces~\cite{kramar2024,zhangwang2026}, shared-prefix execution avoids repeated work~\cite{juravsky2024}, and mixed precision reduces the cost of selected operations~\cite{wang2025}. FIL uses the structure within a finite intervention family: each single-intervention component contributes to an axis candidate and to several joint candidates.

At a selected MLP's gate and up projections, \FIL{} decomposes each candidate's input into a native state, two axis components, and a full-dimensional mixed residual. It projects the axes once at higher precision and reuses the same numerical results in all corresponding candidates. With exact component formation and reconstruction, shared axis errors cancel in the interaction at the projection output. Every candidate retains its mixed residual and original nonlinear continuation.

The contribution comprises an execution algorithm (Section~\ref{sec:method}), a local error analysis including nonlinear propagation and finite-precision reconstruction (Section~\ref{sec:theory}), and evaluation in a pretrained VLM (Section~\ref{sec:experiments}). The matched control gives the axis candidates the same higher precision and uses the same nominal BF16-product budget, but projects each joint candidate independently. FIL improves final interaction fidelity on 16 synthetic families and in two four-image natural studies. In the second natural-image study, which compares region pairs, the two methods have identical native and single-intervention outputs, so their difference lies in the joint candidates.

The benefit is relative to this matched control. Applying multi-product BF16 arithmetic to all candidates (full BF3) gives lower aggregate error at similar total query times, and single-product BF16 (BF1) already solves both tested selection tasks. Whether FIL lowers the cost of a decision that needs finer interaction estimates remains open.

\section{Related Work}
\paragraph{Activation patching and interactions.}
What activation patching measures depends on the choice of intervention, reference, and output metric~\cite{heimersheim2024}. Vaidyanathan et al.~\cite{vaidyanathan2026} analyze interaction terms in mediation-based patching estimands and show how they affect component rankings and circuit interpretation. We take a two-intervention response surface as given and study how accurately it is computed. AtP*~\cite{kramar2024} estimates component effects with gradients, and Zhang and Wang~\cite{zhangwang2026} correct such estimates with downstream Hessian--vector products, including a multi-step variant evaluated at matched compute. These estimators are built to screen many candidates. FIL instead runs the full nonlinear computation of every candidate in a small, fixed family.

\paragraph{Execution reuse and precision allocation.}
Hydragen~\cite{juravsky2024} performs exact attention with shared prefixes and hierarchical reuse; SGLang~\cite{zheng2024} uses RadixAttention to reuse KV caches across calls. CAAP~\cite{izadi2026} parallelizes CLS-token evaluation for its prescribed intervention and approximates blank-context attention by assuming that it is nearly invariant. FIL combines identical-history reuse with shared axis projections inside a finite family. PAHQ~\cite{wang2025} keeps the attention components under investigation at higher precision, and ResQ~\cite{saxena2025} does the same for a high-variance PCA subspace. FIL allocates precision by native, axis, and mixed roles and retains each full-dimensional mixed residual.

\paragraph{Quantization and numerical differences.}
DeltaQuant~\cite{li2026} separates spatiotemporal activation cores from low-precision deltas in video diffusion; CFQ~\cite{yahyati2026} targets counterfactual-recourse stability under quantization. QERA~\cite{zhang2024} and LRC~\cite{scetbon2024} reconstruct quantization errors with low-rank corrections. Multi-product BF16 arithmetic~\cite{henry2019} supplies the BF1/BF3 building blocks used here. Computational divided differencing~\cite{repsrall,vavasis2013} transforms numerical programs so that finite differences, including higher-order and multivariate ones, are evaluated accurately; CD-T~\cite{hsu2024} decomposes transformer contributions for circuit discovery. Baligar~\cite{baligar2026} identifies floating-point floors in finite pair-interaction measurements and argues for higher-precision references. Our contribution is an execution rule that applies these tools to the axis and mixed structure of a finite intervention family.

\section{Problem Formulation}
\label{sec:problem}
Consider a frozen VLM $F$, an image--question pair, an intervention site $\ell$, and two fixed donor differences $D_A,D_B$ supported on designated visual tokens. The initial patched state is
\begin{equation}
 X_{ij}=X_{00}+\alpha_iD_A+\beta_jD_B,
 \quad \alpha_i=\beta_i=i/7.\label{eq:input}
\end{equation}
Here $i,j\in\{0,\ldots,7\}$, so a family has 64 candidates. Donors, supports, token order, positions, attention masks, and all model parameters are fixed within a family. The reference $F^\star$ runs the original model forward from each patched state in its reference arithmetic, which is FP32 in the VLM experiments. The target is thus the model's response to these internal interventions. It need not equal the response to an image with both edits, because adding two donor differences does not in general reproduce the state of a jointly edited image.

\paragraph{Output event.}
Let $\ell_{ij,v}$ be the final-position, next-token logit for vocabulary token $v$. The indices $c=1,\ldots,4$ denote red, green, blue, and yellow, respectively. Each fixed, disjoint set $G_c$ contains the one-token spellings in Appendix~\ref{app:tokens}. With vocabulary $V$, define $G_5=V\setminus\bigcup_{c=1}^4G_c$ as OTHER. Every VLM study in this paper scores all five groups:
\begin{align}
 s_{ij,c}&=\log\sum_{v\in G_c}\exp\ell_{ij,v},\quad c=1,\ldots,5,\label{eq:group}\\
 y_{ij,c}&=\frac{\exp s_{ij,c}}{\sum_{d=1}^{5}\exp s_{ij,d}}.\label{eq:group-prob}
\end{align}
Group log-sum-exp is evaluated in FP64 from the model's FP32 logits. The implementation subtracts the common mean of the five group logits before the softmax, which does not change these probabilities in real arithmetic. The losses average over all five components, including OTHER.

The four-color conditional probability is
\begin{equation}
 p_{ij,c}=\frac{\exp s_{ij,c}}{\sum_{d=1}^{4}\exp s_{ij,d}}
 =\frac{y_{ij,c}}{m_{ij}},\quad c=1,\ldots,4,
 \label{eq:conditional}
\end{equation}
where the total color-group mass is
\begin{equation}
 m_{ij}=\sum_{c=1}^{4}y_{ij,c}
 =\frac{\sum_{v\in\cup_{c=1}^4G_c}\exp\ell_{ij,v}}{\sum_{v\in V}\exp\ell_{ij,v}},
 \qquad y_{ij,5}=1-m_{ij}.
 \label{eq:mass}
\end{equation}
Because a small error in $p$ can coexist with a large error in the color mass $m$, we use the five-group output as the primary VLM metric and report conditional color probabilities separately. We use these next-token events to measure arithmetic fidelity and do not evaluate answer accuracy or free-form generation. Let $y^\star$ denote the reference output and $\widehat y$ the output of an approximate execution. The CPU probe in Appendix~\ref{app:reproduce} uses its own four-output softmax.

\paragraph{Interactions and losses.}
For $i,j>0$, the anchored interaction is
\begin{equation}
 I_{ij}=\Dmix y_{ij}
 =y_{ij}-y_{i0}-y_{0j}+y_{00}.
 \label{eq:interaction}
\end{equation}
These 49 vectors are anchored at the native state. With $e=\widehat y-y^\star$ and $\E$ the 112 undirected horizontal/vertical edges, we measure
\begin{align}
 L_{\rm pair}&=\frac{1}{112K}\sum_{(u,v)\in\E}\norm{e_u-e_v}_2^2,\label{eq:pair}\\
 L_{\rm int}&=\frac{1}{49K}\sum_{i,j=1}^{7}\norm{\Dmix e_{ij}}_2^2,\label{eq:int}\\
 L_{\rm end}&=\frac{1}{64K}\sum_{i,j=0}^{7}\norm{e_{ij}}_2^2,
\end{align}
where $K=5$ for all VLM evaluations and $K=4$ for the synthetic CPU probe. The maximum probability error is $\max_{ijc}|e_{ij,c}|$. The edge set contains 14 axis--axis, 14 axis--joint, and 84 joint--joint edges. Interaction RMS is $\sqrt{L_{\rm int}}$, with pooling across equally sized families where specified.

\paragraph{Query cost.}
One query evaluates the complete intervention family. Its cost includes source acquisition (the native and donor forward passes that produce the donor differences), native and history preparation, component formation, casts, projections, reconstruction, and the nonlinear suffix. Section~\ref{sec:setup} gives the timing and aggregation protocol.

\section{Finite Interaction Lens}
\label{sec:method}
\subsection{Exact-history execution and native centering}
Our exact-history executor reuses cached prefix states among candidates with identical effective histories. Before either intervention enters the history, all candidates share the native prefix. After one enters, candidates with the same intervention share its history. After both enter, each joint candidate keeps its own history. The full-branch frontier is the start of the final segment in the history plan before short segments are merged, where all 64 candidates have distinct histories. Approximate projections begin there. Merging short segments can split branches earlier for efficient execution, but those earlier tokens keep FP32 MLP computation.

Native-relative BF1 computes a selected linear projection as
\begin{equation}
 \widehat Z_{ij}=S_0+\B_1(W,H_{ij}-H_{00}),
 \label{eq:delta}
\end{equation}
where $S_0$ is the cached projection evaluated at native precision and $H_{ij}$ is the \emph{reached input} of candidate $(i,j)$: the normalized input that the candidate's own computation produces at this projection. When earlier layers are also approximated, $H_{ij}$ comes from the approximate trajectory, not from the reference. We call this native-relative form \emph{delta}. The \emph{plain} form applies $\B_p(W,H_{ij})$ directly to each non-native input, without subtracting and restoring the native projection. Both forms retain the native candidate at its reference precision. Thus plain BF3 and delta BF3 differ in centering; \emph{full BF3} refers to either form with BF3 on all 63 non-native candidates.

For definiteness, with $Q$ denoting BF16 rounding, define $h_h=Q(h)$, $h_l=Q(h-h_h)$ and analogous weight components. BF1 uses $W_hh_h$. BF3 forms
\begin{equation}
 \B_3(W,h)=W_hh_h+W_hh_l+W_lh_h,
 \label{eq:bf3}
\end{equation}
with FP32 accumulation and reconstruction in our GPU implementation. BF3 omits the low--low product and still has finite accumulation error. If $W_l=0$, only two nonzero products remain, and an optimized baseline can skip the zero product once the checkpoint has been checked (Section~\ref{sec:setup}).

\subsection{Decomposition at the reached input}
At each selected gate or up projection, define
\begin{align}
 A_i &=H_{i0}-H_{00}, &B_j&=H_{0j}-H_{00},\\
 C_{ij}&=H_{ij}-H_{i0}-H_{0j}+H_{00}.&&\label{eq:components}
\end{align}
In real arithmetic,
\begin{equation}
 H_{ij}=H_{00}+A_i+B_j+C_{ij}.
\end{equation}
FIL keeps every coordinate of every $C_{ij}$. Normalization, attention, and MLP nonlinearities create mixed components even from the additive initial family.

For each of the gate and up projections, the axis projections are computed once,
\begin{equation}
 U_i=\B_{p_A}(W,A_i),\quad V_j=\B_{p_B}(W,B_j),
\end{equation}
and the same stored arrays are reused in
\begin{align}
 \widehat Z_{i0}&=S_0+U_i, &\widehat Z_{0j}&=S_0+V_j,\\
 \widehat Z_{ij}&=S_0+U_i+V_j+\B_{p_C}(W,C_{ij}).&&\label{eq:factored}
\end{align}
Figure~\ref{fig:scheme} shows where this arithmetic sits in the decoder. Proposition~1 below requires that corresponding candidates receive identical axis contributions. Caching guarantees this and also saves work. Recomputing the axis projections would satisfy the requirement as well, but only if operands, dtype, kernel, shape, and reduction order reproduce identical values.

\emph{Factored BF1} sets all three precisions to BF1. \emph{Factored axes BF3}, called FIL in the experiments, uses BF3 for the 14 axes and BF1 for the 49 mixed components. We apply the same choice to gate and up projections. Each candidate then executes its own SiLU, elementwise product, full down projection, and residual update. In a multilayer extension, the next selected layer would decompose its own reached input again. The FIL experiments below approximate a single MLP.

\begin{figure}[!htbp]
\centering
\begin{tikzpicture}[font=\small,>=Stealth,
 box/.style={draw=navy!65,rounded corners=2pt,align=center,inner sep=5pt,minimum height=1.05cm},
 flow/.style={->,thick,draw=navy!75}]
\node[anchor=west,font=\small\bfseries] at (-.1,1.25) {(a) One decoder MLP; all 64 trajectories kept};
\node[box,text width=3.05cm] (h) at (1.6,0) {Reached states $H_{ij}$\\Original attention\\+ normalization};
\node[box,text width=3.05cm,fill=teal!9] (p) at (5.5,0) {FIL gate / up\\$\widehat g_{ij},\widehat u_{ij}$\\at decoder 13};
\node[box,text width=3.1cm] (n) at (9.5,0) {Original SwiGLU\\$\operatorname{SiLU}(\widehat g)\odot\widehat u$\\full down + residual};
\node[box,text width=3.05cm] (o) at (13.3,0) {Original suffix\\Joint-specific KV\\five-group output};
\draw[flow] (h)--(p);\draw[flow] (p)--(n);\draw[flow] (n)--(o);
\node[anchor=west,font=\small\bfseries] at (-.1,-1.45) {(b) Inside each selected gate/up projection};
\node[box,text width=2.9cm,minimum height=4.0cm] (split) at (1.6,-3.85) {Decompose the\\same reached input\\[4pt]$H_{ij}=H_{00}$\\$+A_i+B_j+C_{ij}$\\[4pt]All mixed\\coordinates kept};
\node[box,text width=3.25cm,fill=gray!8,minimum height=.8cm] (base) at (6.1,-2.15) {1 native input\\Original precision: $S_0$};
\node[box,text width=3.25cm,fill=teal!13,minimum height=.9cm] (axis) at (6.1,-3.85) {14 axes (BF3)\\$U_i=\B_3(W,A_i)$\\$V_j=\B_3(W,B_j)$};
\node[box,text width=3.25cm,fill=orange!12,minimum height=.8cm] (mix) at (6.1,-5.55) {49 mixed (BF1)\\$T_{ij}=\B_1(W,C_{ij})$};
\draw[flow] (split.east)--++(.45,0)|-(base.west);
\draw[flow] (split.east)--(axis.west);
\draw[flow] (split.east)--++(.45,0)|-(mix.west);
\node[box,text width=5.25cm,minimum height=4.0cm] (join) at (12,-3.85) {Reconstruct every candidate\\[4pt]
$\widehat Z_{00}=S_0$\\
$\widehat Z_{i0}=S_0+U_i$\\
$\widehat Z_{0j}=S_0+V_j$\\
$\widehat Z_{ij}=S_0+U_i+V_j+T_{ij}$\\[4pt]
\textcolor{teal}{Same stored $U_i,V_j$}\\in axis and joint candidates};
\draw[flow] (base.east)--++(.4,0)-|([xshift=-.45cm]join.north west)--(join.north west);
\draw[flow,teal,line width=1.2pt] (axis.east)--(join.west);
\draw[flow] (mix.east)--++(.4,0)-|([xshift=-.45cm]join.south west)--(join.south west);
\end{tikzpicture}
\caption{FIL within decoder 13. (a) The gate/up projections use structured arithmetic; each candidate retains its nonlinear suffix and KV history. (b) Native, axis, and full mixed projections reconstruct all 64 candidates. The same stored axis projections enter the axis and joint outputs.}
\label{fig:scheme}
\end{figure}

\begin{algorithm}[t]
\caption{Structured projection in a finite family.}
\label{alg:fil}
\small
\hrule\vspace{5pt}
\textbf{Input:} reached $8\!\times\!8$ states, original weights, candidate histories, and the selected projections.\par
\begin{enumerate}
\item Execute original candidate-specific attention and normalization. Reuse a history only where it is identical.
\item For each selected gate/up projection, read the native input and form all $A_i,B_j,C_{ij}$ in FP32.
\item Evaluate $S_0$ once at native precision. Compute and store $U_i,V_j$ with the prescribed axis precision.
\item Project all 49 full mixed components and reconstruct axis/joint outputs using Eq.~\eqref{eq:factored}.
\item Apply original nonlinearities, multiplication, full down projection, and residual to every candidate.
\item Continue the original suffix; repeat the decomposition only at a selected projection. Return all outputs.
\end{enumerate}
\vspace{5pt}\hrule
\end{algorithm}

\subsection{Matched control and operation count}
The matched control, \emph{axes-only BF3}, uses native-relative BF3 on the 14 axis candidates and native-relative BF1 on the full change of each of the 49 joint candidates. FIL uses the same nominal row-product budget, splitting each joint change into axis and mixed components and reusing the axis projections. This comparison tests the combined effect of factoring and numerical reuse beyond higher axis precision alone.

Excluding the common native projection, BF1 requires 63 row-products per token and projection. With generic three-product BF3, FIL's allocation uses $14\cdot3+49=91$, versus $63\cdot3=189$ for full BF3. For our checkpoint the BF16 residual of the gate and up weights is exactly zero, so after skipping that product the counts are 77 for FIL and axes-only BF3 and 126 for full BF3. These counts leave out the work of decomposition and reconstruction, which grows with the grid and feature sizes, and the temporary storage for the axis projections. They therefore do not predict the total query time, which we measure directly (Section~\ref{sec:cost}).

For an $N\times N$ grid and an axis projection costing $p$ row-products, the corresponding counts are $(N-1)^2+2p(N-1)$ for FIL and $p(N^2-1)$ for full higher-precision projection. Their ratio approaches $1/p$ as $N$ grows. Component formation, reconstruction, and the number of nonlinear trajectories still grow quadratically, so a larger grid alone does not ensure a runtime advantage.

\section{Error Analysis}
\label{sec:theory}
\subsection{Local linear cancellation}
For this local result, define $Z_{ij}=WH_{ij}$ in real arithmetic at the same reached inputs used by the approximate projection. This reference differs from the original-model trajectory $WH_{ij}^{\star}$ once earlier approximations have changed the inputs. Define
\begin{align}
 e_0&=S_0-WH_{00},\\
 e_i^A&=\B_{p_A}(W,A_i)-WA_i,\qquad
 e_j^B=\B_{p_B}(W,B_j)-WB_j,\\
 e_{ij}^C&=\B_{p_C}(W,C_{ij})-WC_{ij}.
 \label{eq:projection-errors}
\end{align}
Assume for the identity that components and sums are formed in exact arithmetic, while the component projection routines may be approximate. Set absent components and their projection contributions on the axes to zero.
\begin{samepage}
\begin{proposition}[Shared-error cancellation]
If the same native and axis projection values are reused in every relevant candidate, then
\begin{equation}
 \Dmix(\widehat Z-Z)_{ij}=e^C_{ij}.
 \label{eq:cancel}
\end{equation}
\end{proposition}
\begin{proof}
The joint projection error is $e_0+e^A_i+e^B_j+e^C_{ij}$, the axis errors are $e_0+e^A_i$ and $e_0+e^B_j$, and the native error is $e_0$. Substitution into the four-term difference cancels the common terms.
\end{proof}
\end{samepage}
Independent delta execution instead rounds $A_i+B_j+C_{ij}$ as one operand in each joint, and its error need not split into axis and mixed parts. Accurate axis outputs therefore do not guarantee accurate interactions. Figure~\ref{fig:cancellation} contrasts the two local error structures.

\begin{figure}[!htbp]
\centering
\begin{tikzpicture}[font=\small,>=Stealth]
\node[anchor=west,font=\small\bfseries] at (0,2.3) {(a) Independent delta rounding};
\node[draw=orange!70,fill=orange!7,rounded corners,align=left,text width=6.6cm,inner sep=8pt] (ind) at (3.5,.55) {Joint: $e^{D}_{ij}$ from $\B(W,A_i+B_j+C_{ij})$\\[5pt]
Axes: $e^{D}_{i0}$ and $e^{D}_{0j}$\\[5pt]
$\Dmix e=e^{D}_{ij}-e^{D}_{i0}-e^{D}_{0j}$\\[5pt]
Each error comes from its own rounding.};
\node[anchor=west,font=\small\bfseries] at (7.7,2.3) {(b) Shared native and axis contributions};
\node at (8.4,1.65) {Candidate};
\node at (10.1,1.65) {Native};\node at (11.3,1.65) {Axis A};\node at (12.5,1.65) {Axis B};\node at (13.9,1.65) {Mixed};
\foreach \y/\lab/\sgn in {1.05/ij/+,0.4/i0/-, -.25/0j/-, -.9/00/+}{
 \node[anchor=east] at (8.9,\y) {$\sgn\, e_{\lab}$};
 \node[fill=gray!12,minimum width=1cm,minimum height=.5cm] at (10.1,\y) {$e_0$};}
\foreach \y in {1.05,.4}{\node[fill=teal!15,minimum width=1cm,minimum height=.5cm] at (11.3,\y) {$e_i^A$};}
\foreach \y in {1.05,-.25}{\node[fill=teal!15,minimum width=1cm,minimum height=.5cm] at (12.5,\y) {$e_j^B$};}
\node[fill=orange!15,minimum width=1cm,minimum height=.5cm] at (13.9,1.05) {$e_{ij}^C$};
\draw[thick] (9.55,-1.3)--(14.5,-1.3);
\node at (10.1,-1.65) {$0$};\node at (11.3,-1.65) {$0$};\node at (12.5,-1.65) {$0$};
\node[draw=orange,rounded corners,inner sep=3pt] at (13.9,-1.65) {$e_{ij}^C$};
\node[anchor=west,align=left,text width=14.8cm] at (0,-2.6) {With finite-precision formation and reconstruction (Eq.~\eqref{eq:finite-error}):
$\displaystyle \Dmix(\widehat Z-WH)=q^C+W(\widetilde C-C)+\Dmix r$.};
\end{tikzpicture}
\caption{Local projection-error structure. Under exact component formation and reconstruction (Proposition~1), shared native and axis errors cancel in the mixed difference. The bottom row shows the additional finite-precision terms from Eq.~\eqref{eq:finite-error}.}
\label{fig:cancellation}
\end{figure}

\FloatBarrier
\subsection{A conditional nonlinear bound}
Let $g$ be a differentiable nonlinear map acting on the projection output; gate and up vectors can be concatenated into its input. At one joint, denote the exact projected components by $z,a,b,c$ and the shared errors by $e_a,e_b,e_c$. For this result, the native projection $z$ is exact. Let $\epsilon_a=\norm{e_a}$, and similarly for $b,c$. Suppose $\norm{Dg}\le L_1$ and $Dg$ is $L_2$-Lipschitz on a convex region containing all evaluation and perturbation paths.
Define the interaction error explicitly as
\begin{align}
 E_g={}&g(z+a+b+c+e_a+e_b+e_c)-g(z+a+b+c)\nonumber\\
 &-g(z+a+e_a)+g(z+a)-g(z+b+e_b)+g(z+b).
 \label{eq:eg}
\end{align}
The native terms cancel because the native projection is exact in this result.
\begin{proposition}[Leakage through a smooth local map]
The interaction error $E_g$ obeys
\begin{align}
 \norm{E_g}\le{}&L_1\epsilon_c\nonumber\\
 &+L_2\epsilon_a\left(\norm{b+c}+\frac{\epsilon_b+\epsilon_c}{2}\right)\nonumber\\
 &+L_2\epsilon_b\left(\norm{a+c}+\frac{\epsilon_a+\epsilon_c}{2}\right).
 \label{eq:bound}
\end{align}
\end{proposition}
The proof is given in Appendix~\ref{app:proof}. The bound shows that axis error is multiplied by variation along the other intervention, whereas mixed-component error enters directly.

\begin{corollary}[Small-amplitude order]
If $\norm{a},\norm{b}=O(t)$, $\norm{c}=O(t^2)$, $\epsilon_a,\epsilon_b=O(ut)$, and $\epsilon_c=O(ut^2)$, with bounded $L_1,L_2$, then $\norm{E_g}=O(ut^2+u^2t^2)$.
\end{corollary}
Here $u$ is a local error scale of the low-precision projection; it depends on dimension and conditioning. Independent candidate rounding can instead introduce an $O(ut)$ contribution to the interaction error. 

\subsection{Formation, reconstruction, and trajectory error}
The cancellation identity is local. Let $\widetilde C_{ij}$ be the mixed component actually formed in finite precision, let $q^C_{ij}=\B_{p_C}(W,\widetilde C_{ij})-W\widetilde C_{ij}$, and let $r_{ij}$ be the roundoff of the reconstruction sum relative to the exact sum of the stored contributions. If the native and axis contributions are identical wherever reused, then, with $\Dmix$ evaluated in real arithmetic,
\begin{equation}
 \Dmix(\widehat Z-WH)_{ij}
 =q^C_{ij}+W(\widetilde C_{ij}-C_{ij})+\Dmix r_{ij}.
 \label{eq:finite-error}
\end{equation}
Evaluating the four-point difference in floating point adds its own error $\rho^{\Delta}_{ij}$. Comparison with the original-model trajectory adds a further term,
\begin{equation}
 \Dmix(\widehat Z-WH^\star)_{ij}
 =\Dmix(\widehat Z-WH)_{ij}+W\Dmix(H-H^\star)_{ij}.
 \label{eq:trajectory-error}
\end{equation}
Rounding errors made while forming $A$ and $B$ are part of the shared axis contributions, so they cancel in the linear four-point difference. Their effects on nonlinear propagation are separate from this cancellation.

FP32 component subtraction can introduce an absolute floor on the order of $u_{32}\norm{H_{00}}$, up to dimension- and operation-dependent constants. Final reconstruction, nonlinear evaluation, and output differencing add their own floors. Native error cancels in the linear identity but not, in general, after a nonlinearity, which is why Proposition~2 assumes an exact native projection. A small-amplitude law for the whole VLM would need assumptions on how errors propagate through it, which this local result does not provide.

Large $C_{ij}$, high curvature, near-singular normalization inputs, and accumulated candidate-specific attention errors can all defeat the intended advantage. The reference execution is itself finite-precision, and a repeatable result can still be systematically inaccurate~\cite{baligar2026}. Agreement between replays therefore does not estimate the precision floor, and ranking interactions near that floor requires a higher-precision reference or an independently verified error bound.

\section{Experiments}
\label{sec:experiments}
\subsection{Setup}
\label{sec:setup}
We evaluate Qwen3-VL-4B-Instruct~\cite{qwen2025}, which has hidden width 2560 and MLP width 9728. Donor differences are inserted on their supports at the output of decoder 12, and the approximations act on the gate and up projections of decoder 13 (zero-based). The rest of each of the 64 candidates runs as in the original model, including attention with candidate-specific KV histories, the full down projection, and the residual stream. Intervention strengths are $0,1/7,\ldots,1$.

\paragraph{Inputs.}
The synthetic set has 16 families (indices 208--223). Each image shows four labeled colored panels on a gray checkerboard, and the question asks ``What color is panel \emph{L}? Answer with one color word.'' Donor A recolors the queried panel, and donor B, built separately from the same native image, recolors the next labeled panel. Their decoder-12 differences are restricted to disjoint 25-token geometric supports. All inputs have 218 tokens. Depending on where the edited panels lie, the full-branch frontier falls at token 40, 117, or 124, so FIL acts on the last 178, 101, or 94 tokens. Appendix~\ref{app:synthetic} gives the renderer, RGB values, donors, questions, and supports.

The first natural-image study uses four COCO val2014 images~\cite{lin2014} (indices 0, 16, 32, and 48) with two object-color questions each. The images come from a four-MLP baseline study of existing precision methods (Appendix~\ref{sec:natural-baseline}). Each donor blends one object's mask halfway toward a fixed palette color, red for A and blue for B on all four images. The unit of analysis is the image, and its two questions are averaged before pooling. A fifth image (index 61) is excluded from the averages and used only to recheck the one gate failure of the baseline study (Section~\ref{sec:natural-input}). Appendix~\ref{app:natural} specifies images, masks, question wording, and supports. A second natural-image study, the region-pair study of Section~\ref{sec:region-pairs}, uses four other images with one question and six region pairs each.

\paragraph{Arithmetic and reference.}
The reference is ordinary full-branch FP32 execution of the checkpoint. \emph{History dense} runs the exact-history executor of Section~\ref{sec:method} with all projections in FP32, so its distance from the reference measures the numerical effect of the execution schedule alone. In the approximate methods, the selected gate/up products use BF16 operands with FP32 accumulation and output, with TF32 and reduced-precision BF16 reductions disabled. The gate and up weights of this checkpoint are exactly representable in BF16, so every BF3 implementation skips the zero weight-residual product. This leaves 77 nominal row-products per token and projection for FIL and axes-only BF3 and 126 for full BF3. Every BF16 cast is audited for nonzero values that round to zero, subnormal values, and nonfinite values. An unsafe cast or a nonfinite low-precision product triggers FP32 recomputation of the affected projection, and a nonfinite source or formed operand stops the run. The audit results can be read back from the GPU check by check (\emph{original}) or in one batch (\emph{batched}); both give identical outputs (Appendix~\ref{sec:audit-batching}). Experiments run on an NVIDIA GH200 120GB with PyTorch 2.11.0+cu130, and Appendix~\ref{app:tokens} lists the model, tokenizer, and backend versions. All errors below are distances to the FP32 reference.

\paragraph{Timing and aggregation.}
Each method is timed for seven repetitions after two warmup runs. We take the median within each family and report the mean of these medians; for natural images, the two question medians are first averaged within each image. Synthetic timing uses families 208, 210, 212, and 214, which the implementation studies also reuse. Source acquisition is measured once per family or natural-image case and added to every repetition of every method, and all other query work, including synchronization and fallbacks, is timed directly. Times are compared only within a run. The storage-reuse study has within-family interquartile ranges of 0.24--11.78\,ms (Appendix~\ref{sec:storage-reuse}). Time differences below are within-run observations.

All main single-MLP comparisons pass two fixed fidelity gates, a maximum probability error of at most 0.01 and an endpoint MSE of at most $10^{-4}$, and no arithmetic fallback was recorded. Because the gates do not separate the methods, we compare them with the finer interaction losses of Section~\ref{sec:problem}. All inputs in this section were also used during method development, and there is no held-out test set. The bootstrap intervals, which resample families or image clusters, are therefore descriptive.

\subsection{Synthetic families}
\label{sec:synthetic-main}
On the 16 synthetic families, FIL lowers the mean final interaction MSE from $5.02\times10^{-10}$ to $2.81\times10^{-11}$ relative to axes-only BF3 and improves every family (Table~\ref{tab:synthetic-main}, Figure~\ref{fig:synthetic-main_cost}). This is a 94.4\% (17.9-fold) reduction in MSE and a 4.23-fold reduction in pooled RMS. The paired mean difference is $-4.74\times10^{-10}$, with a 95\% bootstrap interval of $[-8.81\times10^{-10},-1.56\times10^{-10}]$. The error also falls on the 84 joint--joint edges, from $1.07\times10^{-9}$ to $5.58\times10^{-11}$, and in the conditional four-color interaction MSE, from $6.24\times10^{-10}$ to $3.52\times10^{-11}$.

Neither factoring nor higher axis precision accounts for the gain alone. Factored BF1, which factors and shares the axes without raising their precision, improves on delta BF1 in 11 of 16 families but raises mean MSE by 21.9\%, with a paired interval spanning zero. Higher axis precision without factoring is the axes-only control itself. In this comparison, the mean gain appears only when factoring with shared axis contributions is combined with higher axis precision.

\begin{table}[t]
\centering\small
\caption{Synthetic single-MLP evaluation: mean interaction MSE over 16 families, and mean of seven-repeat medians over four timed families relative to history dense (1671.2\,ms). Delta BF1 ($14+49$) and delta BF1 (63 packed) are two packings of the same computation; the BF3 rows use the optimized implementations.}
\label{tab:synthetic-main}
\begin{tabular}{lrr}\toprule
Method & Mean interaction MSE & Time vs.\ history dense (ms)\\\midrule
History dense & $1.77\times10^{-12}$ & 0.0\\
Delta BF1 ($14+49$) & $7.97\times10^{-10}$ & $-4.0$\\
Axes-only BF3 & $5.02\times10^{-10}$ & $-2.5$\\
Factored BF1 & $9.72\times10^{-10}$ & $-1.9$\\
FIL (factored axes BF3) & $2.81\times10^{-11}$ & $-1.0$\\
Delta BF3 & $3.53\times10^{-12}$ & $-2.3$\\
Plain BF3 & $2.13\times10^{-12}$ & $-2.1$\\
Delta BF1 (63 packed) & $7.97\times10^{-10}$ & $-6.4$\\\bottomrule
\end{tabular}
\end{table}

\begin{figure}[!htbp]
\centering
\includegraphics[width=\linewidth]{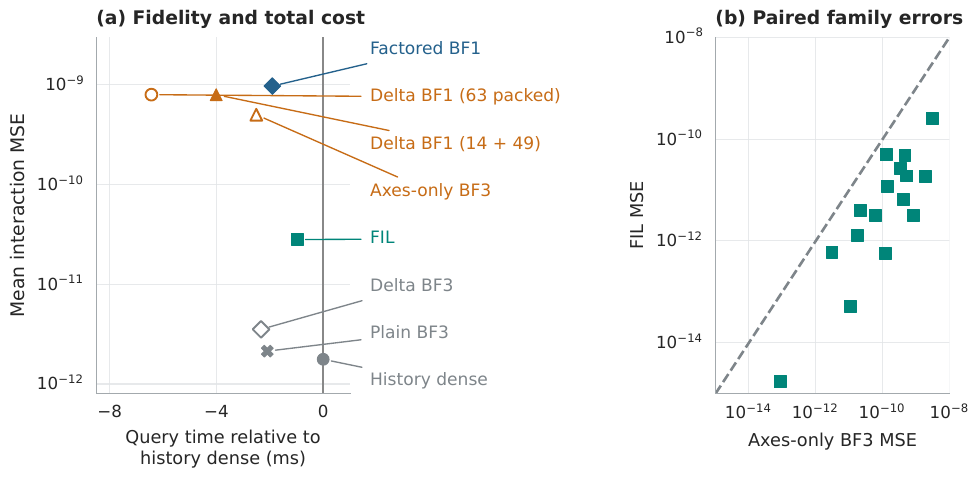}
\caption{Synthetic evaluation. (a) Mean MSE over 16 families versus mean query time over four family medians. Time is relative to history dense (1671.2\,ms); negative values indicate lower total query time. (b) Paired family losses on equal logarithmic scales. The dashed diagonal marks parity.}
\label{fig:synthetic-main_cost}
\end{figure}

FIL adds 1.5\,ms relative to axes-only BF3 on a roughly 1.67\,s query. Compared with delta and plain BF3, FIL has 7.95 and 13.16 times their MSE and takes 1.4 and 1.1\,ms longer. Section~\ref{sec:cost} compares implementation variants.

\subsection{Natural images}
\label{sec:natural-input}
On the four natural images, FIL lowers the image-mean interaction MSE by 99.85\% relative to axes-only BF3 and improves every image and every question (Table~\ref{tab:natural-input}). On the 84 joint--joint edges, the error falls by 99.84\%, from $1.46\times10^{-8}$ to $2.40\times10^{-11}$.

\begin{table}[!htbp]
\centering\small
\caption{Natural-image evaluation: four image means, each averaging two questions. Time averages seven-repeat case medians within image, then across images, and is given relative to history dense (1405.4\,ms). The BF3-based methods use batched audit read-back; Table~\ref{tab:natural-original-audit} lists both read-back variants. Image 61 is not included.}
\label{tab:natural-input}
\begin{tabular}{lrr}\toprule
Method & Mean interaction MSE & Time vs.\ history dense (ms)\\\midrule
History dense & $6.96\times10^{-12}$ & 0.0\\
Delta BF1 & $8.34\times10^{-9}$ & $-7.0$\\
Axes-only BF3 & $7.72\times10^{-9}$ & $-6.1$\\
FIL & $1.16\times10^{-11}$ & $-5.3$\\
Delta BF3 & $7.71\times10^{-12}$ & $-4.9$\\
Plain BF3 & $1.08\times10^{-11}$ & $-5.4$\\
\bottomrule\end{tabular}
\end{table}

Most of the absolute gain comes from one image. Image 48 accounts for 99.59\% of the total decrease in MSE, although the reduction over the other three images is still 96.77\% (Table~\ref{tab:natural-input_images}). Image 16 has errors of order $10^{-16}$ under both methods, so its relative reduction of 41.6\% is a negligible absolute change.

\begin{table}[!htbp]
\centering\footnotesize
\caption{Image-level interaction MSE, FIL versus axes-only BF3 with batched auditing. Share is the fraction of the total absolute MSE reduction; Omit is the relative reduction over the other three images.}
\label{tab:natural-input_images}
\begin{tabular}{rrrrrr}\toprule
Image & Axes-only MSE & FIL MSE & Reduction (\%) & Share (\%) & Omit: reduction (\%)\\\midrule
0 & $8.11\times10^{-11}$ & $3.50\times10^{-12}$ & 95.69 & 0.252 & 99.86\\
16 & $4.90\times10^{-16}$ & $2.87\times10^{-16}$ & 41.56 & $6.61\times10^{-7}$ & 99.85\\
32 & $4.94\times10^{-11}$ & $7.20\times10^{-13}$ & 98.54 & 0.158 & 99.85\\
48 & $3.08\times10^{-8}$ & $4.23\times10^{-11}$ & 99.86 & 99.59 & 96.77\\
\bottomrule\end{tabular}
\end{table}

\begin{figure}[!htbp]
\centering
\includegraphics[width=\linewidth]{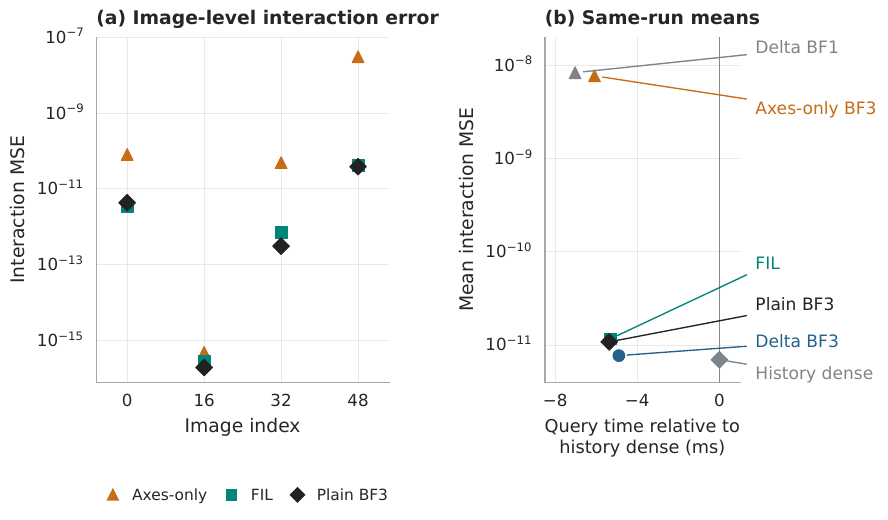}
\caption{Natural-image evaluation with batched auditing. (a) MSE for each image, averaged over two questions. (b) Mean error and query time relative to history dense (1405.4\,ms); negative times indicate lower total cost. Original-audit variants are listed in Table~\ref{tab:natural-original-audit}.}
\label{fig:natural-input}
\end{figure}

FIL has 1.075 times the MSE of plain BF3 and 1.51 times that of delta BF3. Its paired time difference from plain BF3 is $+0.05$\,ms per roughly 1.4\,s query, with FIL faster on two images and a 95\% interval of $[-0.72,+0.80]$\,ms. Against delta BF3, the observed saving is 0.42\,ms.

In the four-MLP baseline study, delta BF1 failed the 0.01 gate on one case (case 123, on image 61) with a maximum error of 0.0132 (Appendix~\ref{sec:natural-baseline}). With only one MLP approximated, the maximum error on this case is 0.00836 under delta BF1 and $3.02\times10^{-5}$ under FIL, both within the gate. Because the baseline study approximated four MLPs, this result does not show that FIL would remove the original failure.

\subsection{Region pairs with identical axis outputs}
\label{sec:region-pairs}
The region-pair study uses four other COCO val2014 images, each with one object-color question and four donor regions. Each of the six unordered region pairs receives the full $8\times8$ amplitude grid. To give all pairs of an image the same precision range, the gate/up approximation in decoder 13 begins at $F=\max_p f_p$, where $f_p$ is the full-branch frontier of pair $p$, and computation before $F$ stays dense. Each pair still keeps its own history and nonlinear continuation. Appendix~\ref{app:region-inputs} specifies the stimuli and execution.

\paragraph{A control with identical native and axis outputs.}
In all 24 pairs, FIL and axes-only BF3 produce bitwise-identical five-group logits at the native state and at all 14 single-intervention states. In real arithmetic, the difference between their interactions is therefore the difference between their joint outputs,
\begin{equation}
 I^{\rm FIL}_{ij}-I^{\rm axes}_{ij}
 =y^{\rm FIL}_{ij}-y^{\rm axes}_{ij},\qquad i,j>0.
 \label{eq:matched-axes}
\end{equation}
The fidelity difference is therefore confined to the joint candidates. The methods also use the same multiply--accumulate counts and packed projection shapes (Appendix~\ref{app:region-inputs}). Their joint operands, reconstruction, and rounding coupling change together.

\paragraph{Fidelity and total bundle cost.}
Losses are averaged over six pairs within each image and then equally across the four images. A timed bundle evaluates all six pairs. After two warmups, we take each image's seven-repeat median and average these medians. Source acquisition (one native and four donor forward passes) and common-frontier planning are measured once per image and added once per bundle. Table~\ref{tab:region-summary} gives differences from the 7.75\,s history-dense bundle, rather than the single-family costs of Table~\ref{tab:natural-input}.

\begin{table}[!htbp]
\centering\small
\caption{Region-pair study: four images, six pairs per image, one MLP. Approximate methods use the common per-image precision range. MSE is an image mean of pair means. Time is relative to history dense (7749.9\,ms), using the mean of seven-repeat image medians for six-pair bundles. Time differences are calculated before rounding. Max error is over all states and five groups.}
\label{tab:region-summary}
\begin{tabular}{lrrr}\toprule
Method & Interaction MSE & Max probability error & Time vs.\ history dense (ms)\\\midrule
History dense & $2.99\times10^{-12}$ & $2.26\times10^{-5}$ & 0.0\\
Delta BF1 & $5.37\times10^{-9}$ & $1.24\times10^{-3}$ & $-14.9$\\
Axes-only BF3 & $3.97\times10^{-9}$ & $1.24\times10^{-3}$ & $-7.3$\\
FIL & $2.57\times10^{-11}$ & $1.72\times10^{-4}$ & $-3.5$\\
Delta BF3 & $3.40\times10^{-12}$ & $2.40\times10^{-5}$ & $-2.0$\\
Plain BF3 & $2.89\times10^{-12}$ & $2.39\times10^{-5}$ & $-6.0$\\
Ordinary exact & Reference & --- & $+2675.6$\\\bottomrule
\end{tabular}
\end{table}

FIL lowers the interaction MSE by 99.35\% relative to axes-only BF3 and improves all four images and all 24 pairs, while taking 3.7\,ms longer per bundle of about 7.7\,s. As in the first natural-image study, one image dominates the absolute numbers. Image 0 (COCO 529850) contributes 94.53\% of the absolute reduction, and its region pair $(0,1)$ alone accounts for 80.05\% of the MSE that remains under FIL. Without image 0, the reduction over the other three images is 99.58\%. The size of the error also depends on whether a pair contains the queried region: the mean FIL MSE is $5.14\times10^{-11}$ over the 12 pairs that do and $2.63\times10^{-18}$ over the 12 pairs with context regions only.

\begin{table}[!htbp]
\centering\small
\caption{Region-pair study image means, each over six pairs. These indices refer to a different inventory from the two-question study in Table~\ref{tab:natural_inputs}; COCO IDs identify the images. Share is the fraction of total axes-only-minus-FIL MSE reduction.}
\label{tab:region-images}
\begin{tabular}{rrrrrr}\toprule
Index & COCO ID & Axes-only MSE & FIL MSE & Plain BF3 MSE & Share (\%)\\\midrule
0 & 529850 & $1.50\times10^{-8}$ & $9.92\times10^{-11}$ & $6.20\times10^{-12}$ & 94.53\\
3 & 302789 & $7.48\times10^{-10}$ & $2.87\times10^{-12}$ & $4.20\times10^{-12}$ & 4.72\\
13 & 315577 & $1.20\times10^{-10}$ & $7.42\times10^{-13}$ & $1.17\times10^{-12}$ & 0.75\\
22 & 221911 & $6.37\times10^{-16}$ & $1.87\times10^{-17}$ & $2.18\times10^{-17}$ & $3.92\times10^{-6}$\\
\bottomrule\end{tabular}
\end{table}

FIL has 8.89 times the MSE of plain BF3 and 7.57 times that of delta BF3, although it has lower MSE than plain BF3 on three of the four images; image 0 decides the aggregate ordering (Figure~\ref{fig:region-pairs}). The mean FIL bundle is 2.5\,ms slower than plain BF3 and 1.5\,ms faster than delta BF3. Both differences are small next to the spread of repeated FIL bundles on a single image, which is 18--79\,ms depending on the image (Appendix~\ref{app:region-inputs}).

\begin{figure}[!htbp]
\centering
\includegraphics[width=\linewidth]{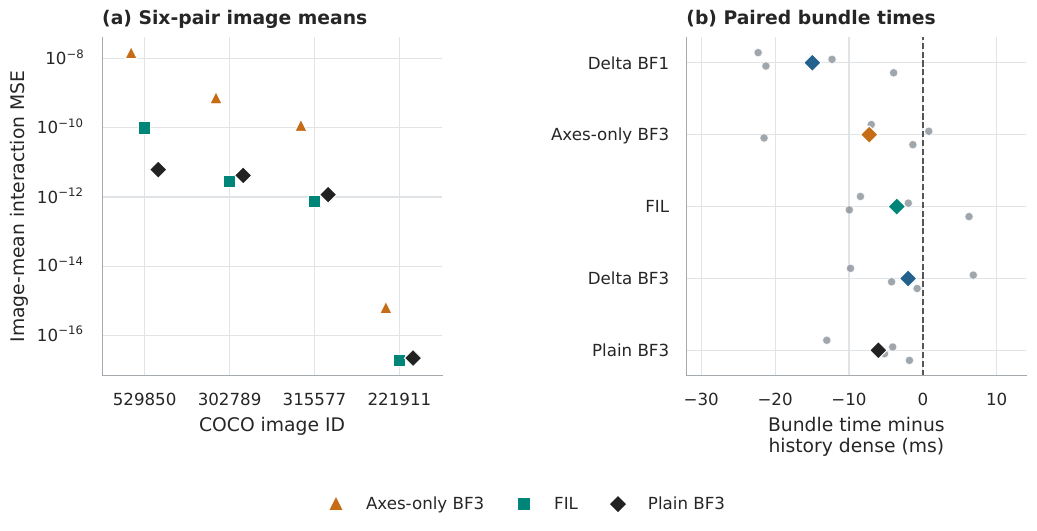}
\caption{Region-pair study at a common precision range. (a) Image-mean interaction MSE over six pairs, labeled by COCO ID. (b) Each dot is an image's seven-repeat median time minus its same-run history-dense median; diamonds are means of four paired differences. Mean history-dense bundle time is 7.75\,s.}
\label{fig:region-pairs}
\end{figure}

\paragraph{Reuse across pairs.}
Sharing an axis across different region pairs requires a separate equivalence check. With the common precision range, separately executed copies of the same native or single-region state differ after normalization in 91 of 116 groups, by at most $1.79\times10^{-6}$ (Appendix~\ref{app:region-inputs}). These differences precede the selected MLP. Proposition~1 instead concerns reuse of a stored axis projection within one family. Cross-pair reuse would need a common execution path and its own fidelity evaluation.

\subsection{Propagation to the final output}
\label{sec:ablations}
Against FP32 local references computed at the same reached inputs, the ratio of FIL to axes-only interaction MSE is 0.0260 at the gate output, 0.0280 at the up output, 0.0262 after SwiGLU, and 0.0265 after the full MLP, and all 16 families improve at every stage. At the final output the ratio is 0.0560 (Figure~\ref{fig:propagation}). The rest of the network thus roughly halves the local advantage, from about 38-fold to about 18-fold, but does not remove it.

To describe the final error, let $e_{ij,k}=\widehat y_{ij,k}-y^\star_{ij,k}$ and define
\begin{equation}
 v_{ij,k}=(e_{ij,k},e_{i0,k},e_{0j,k},e_{00,k})^\top,
 \quad G=\frac{1}{49K}\sum_{i,j=1}^7\sum_{k=1}^K v_{ij,k}v_{ij,k}^\top.
 \label{eq:moment}
\end{equation}
For $s=(1,-1,-1,1)^\top$,
\begin{equation}
 L_{\rm int}=s^\top Gs
 =\operatorname{tr}(G)+\sum_{a\ne b}s_as_bG_{ab}.
 \label{eq:moment-split}
\end{equation}
$G$ is an uncentered second moment. Its diagonal weights states by their occurrence in anchored tuples, whereas endpoint MSE weights the 64 states equally.

\begin{table}[t]
\centering\small
\caption{Final-output uncentered error moments, averaged over 16 synthetic families. The cross sum includes the contrast signs from Eq.~\eqref{eq:moment-split}.}
\label{tab:moments}
\begin{tabular}{lrrr}\toprule
Method & Diagonal sum & Signed cross sum & Interaction MSE\\\midrule
Axes-only BF3 & $5.09\times10^{-10}$ & $-7.43\times10^{-12}$ & $5.02\times10^{-10}$\\
FIL & $2.89\times10^{-11}$ & $-8.03\times10^{-13}$ & $2.81\times10^{-11}$\\\bottomrule
\end{tabular}
\end{table}

The diagonal sum falls by $4.80\times10^{-10}$, which is 101.4\% of the net reduction, and the change in the cross terms offsets 1.4\%. At the final output, the gain therefore appears as smaller errors in the individual states rather than as stronger cancellation between them. Proposition~1 concerns the projection errors before the nonlinear suffix, so the two observations are compatible. Because Eq.~\eqref{eq:moment-split} is an accounting identity, it cannot tell whether the smaller state errors come from operand scale, shared rounding, or caching.

\begin{figure}[!htbp]
\centering
\includegraphics[width=\linewidth]{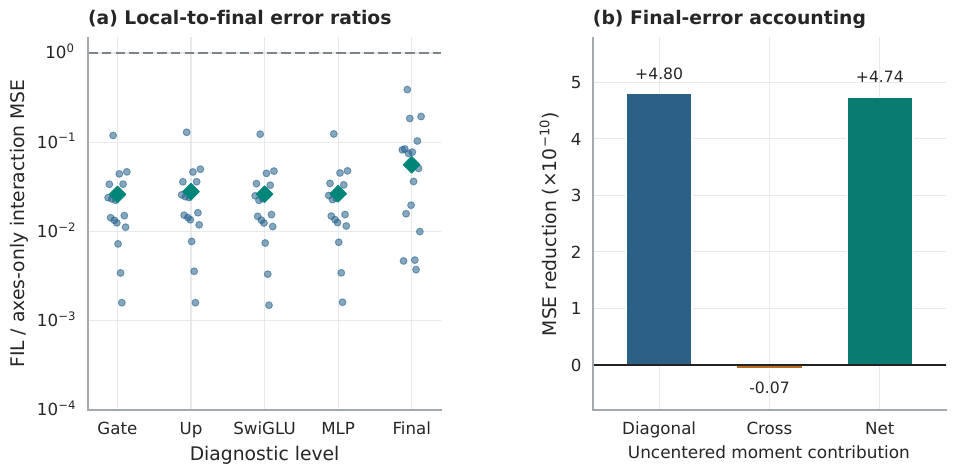}
\caption{Error propagation on 16 synthetic families. (a) Each point is a family and each diamond is the ratio of mean losses, both against FP32 references; the dashed line marks parity. (b) Axes-only minus FIL in the diagonal, signed cross, and total terms of Eq.~\eqref{eq:moment-split}.}
\label{fig:propagation}
\end{figure}

\subsection{Precision allocation and caching}
\label{sec:allocation}
On the four-family subset, moving the BF3 budget from the 14 axes to 14 mixed components raises MSE by factors of 33.17 for fixed positions and 36.23 for norm-based selection (Table~\ref{tab:ablation}, Figure~\ref{fig:ablation}). Both variants take longer. Assigning BF3 to every factored component or keeping one projection dense improves fidelity at higher cost.

Recomputing the axis projections with the same shapes gives identical logits and adds 5.7\,ms. Caching saves this computation, while both implementations preserve identical numerical axis contributions.

\begin{table}[t]
\centering\small
\caption{Ablations on four synthetic families, relative to FIL measured in the same run (MSE $8.03\times10^{-11}$). The two mixed-14 controls keep the nominal BF3 budget; the final three rows change it or the set of approximated projections.}
\label{tab:ablation}
\begin{tabular}{lrrr}\toprule
Variant & Mean interaction MSE & MSE / FIL & Time $-$ FIL (ms)\\\midrule
FIL, cached & $8.03\times10^{-11}$ & 1.000 & 0.0\\
FIL, same-shape axis recomputation & $8.03\times10^{-11}$ & 1.000 & $+5.7$\\
BF3 on fixed mixed 14 & $2.66\times10^{-9}$ & 33.17 & $+1.4$\\
BF3 on norm-selected mixed 14 & $2.91\times10^{-9}$ & 36.23 & $+2.5$\\\midrule
Factored all BF3 & $1.33\times10^{-12}$ & 0.0165 & $+3.5$\\
FIL at gate only, up dense & $2.93\times10^{-11}$ & 0.365 & $+5.3$\\
FIL at up only, gate dense & $4.59\times10^{-11}$ & 0.571 & $+4.9$\\\bottomrule
\end{tabular}
\end{table}

\begin{figure}[!htbp]
\centering
\includegraphics[width=\linewidth]{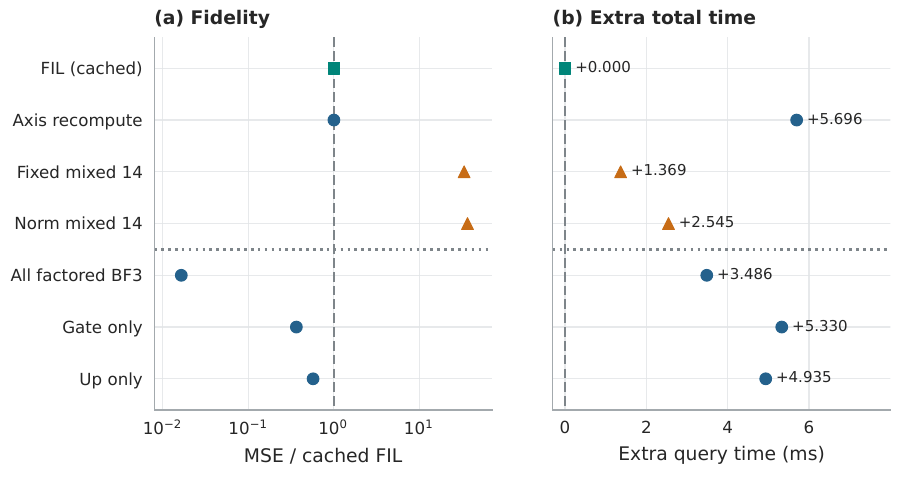}
\caption{Precision and caching ablations on four families. Both panels use the same-run FIL baseline in Table~\ref{tab:ablation}. Rows below the divider change the precision or projection budget. Time differences average within-family seven-repeat medians.}
\label{fig:ablation}
\end{figure}

At ordinary amplitudes, component formation and reconstruction contribute little to the local projection error. Across four families, both projections, three subtraction orders, and two precisions for each operation, the ratios of their squared errors to the total local interaction error are about $10^{-8}$ and $10^{-6}$, respectively. Appendix~\ref{app:formation} defines this reference and gives the ratios. The CPU experiment in Appendix~\ref{app:reproduce} examines the small-amplitude regime, where a numerical floor becomes visible.

\subsection{Fidelity and total query time}
\label{sec:cost}
FIL uses 77 nominal row-products per token and projection against 126 for full BF3. With temporary-storage reuse on the 16-family set, it retains the 94.4\% MSE reduction over axes-only BF3 at an additional 1.8\,ms. Its MSE is 13.16 times that of plain BF3, which is 1.0\,ms faster. Appendix~\ref{sec:storage-reuse} gives the comparison and the sensitivity to timing aggregation.

The four-family implementation comparison combines index reuse and fused reconstruction with two audit read-back strategies (Table~\ref{tab:audit-batching}). Batched FIL has 90.37\% lower MSE than batched axes-only BF3 and adds 0.58\,ms. Relative to the faster plain BF3 variant, it saves 0.67\,ms with 41.55 times the MSE. The difference between 90.37\% and 94.4\% comes from averaging over four rather than 16 families; on the four common families, the accuracy metrics match those of the 16-family run. At equal accuracy, original-audit FIL is 0.08\,ms faster than batched FIL in this run.

\begin{table}[!htbp]
\centering\small
\caption{Audit comparison on four synthetic families: mean interaction MSE and mean of within-family seven-repeat median query times, relative to history dense (1636.3\,ms). Both audit variants are measured in the same run.}
\label{tab:audit-batching}
\begin{tabular}{lrr}\toprule
Method & Mean interaction MSE & Time vs.\ history dense (ms)\\\midrule
History dense & $1.50\times10^{-12}$ & 0.0\\
Delta BF1 & $9.69\times10^{-10}$ & $-7.2$\\
Axes-only BF3, original audit & $8.34\times10^{-10}$ & $-5.1$\\
FIL, original audit & $8.03\times10^{-11}$ & $-5.0$\\
Delta BF3, original audit & $1.48\times10^{-12}$ & $-3.2$\\
Plain BF3, original audit & $1.93\times10^{-12}$ & $-4.2$\\
Axes-only BF3, batched audit & $8.34\times10^{-10}$ & $-5.5$\\
FIL, batched audit & $8.03\times10^{-11}$ & $-4.9$\\
Delta BF3, batched audit & $1.48\times10^{-12}$ & $-4.0$\\
Plain BF3, batched audit & $1.93\times10^{-12}$ & $-4.1$\\
\bottomrule\end{tabular}
\end{table}

\begin{figure}[!htbp]
\centering
\includegraphics[width=\linewidth]{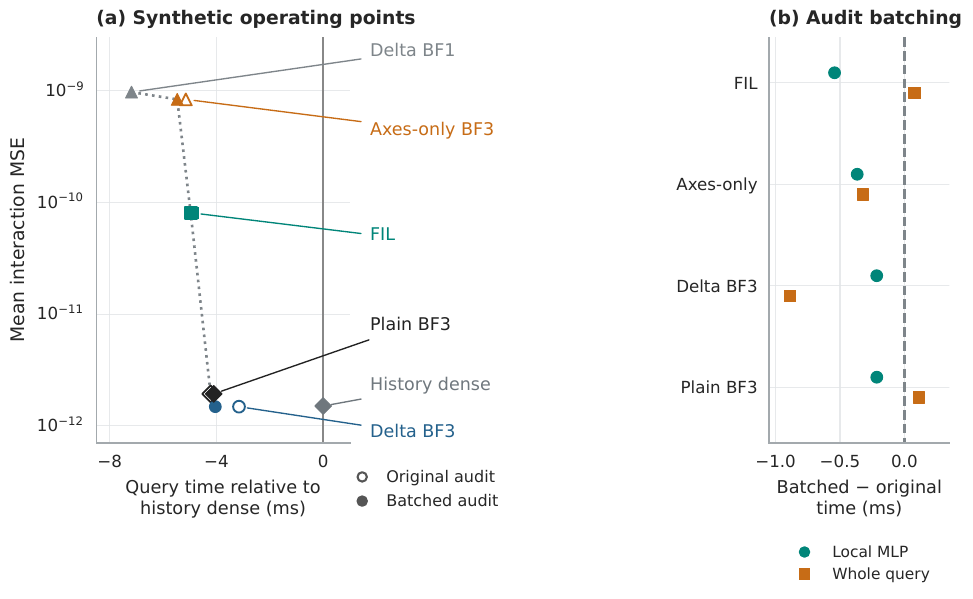}
\caption{Audit comparison on four synthetic families. (a) Same-run means from Table~\ref{tab:audit-batching}, with time relative to history dense (1636.3\,ms). Open and filled markers distinguish original and batched auditing; one label identifies each method. Dotted segments connect observed nondominated points. (b) Batched minus original audit time, measured separately for the local MLP and whole query.}
\label{fig:audit-batching_cost}
\end{figure}

Batching halves FIL's synchronizations from 24 to 12 and saves 0.54\,ms inside the MLP, while the whole-query time increases by 0.08\,ms. The measurements do not identify what offsets the local saving (Appendix~\ref{sec:audit-batching}). Across the synthetic and natural studies, lower nominal projection cost has not produced a consistent query-time advantage over full BF3.

\subsection{Candidate selection}
\label{sec:selection}
We next ask whether the smaller error changes a decision made from the interactions. As a post hoc test, we select, for each of the 16 synthetic families and five output groups, the amplitude pair with the largest absolute interaction among the 49, breaking ties in row-major order. We do this for all 13 methods of the storage-reuse run (Appendix~\ref{sec:storage-reuse}). Every method, including delta BF1, selects the same pair as the reference in all 80 family--group comparisons, with top-5 recall of 1 and zero regret.

To assess the separation of candidate scores, define $\varepsilon=\max_c|\widehat I_c-I_c^\star|$ and let $g$ be the gap between the largest and second-largest reference scores $|I_c^\star|$. Since $\bigl||\widehat I_c|-|I_c^\star|\bigr|\leq\varepsilon$, when all candidates have been evaluated,
\begin{equation}
 g>2\varepsilon\quad\Longrightarrow\quad
 \arg\max_c|\widehat I_c|=\arg\max_c|I_c^\star|.
 \label{eq:selection-gap}
\end{equation}
The condition is sufficient but not necessary.

\begin{table}[!htbp]
\centering\small
\caption{Post hoc selection margins on the 16 synthetic families and five output groups, relative to the FP32 reference, for three methods of the storage-reuse run (Appendix~\ref{sec:storage-reuse}). All 13 methods of that run, including the ten not listed, select correctly in all 80 comparisons.}
\label{tab:selection}
\begin{tabular}{lrr}\toprule
Method & Comparisons with $g>2\varepsilon$ & Observed top-1 agreement\\\midrule
Delta BF1 & 75/80 & 80/80\\
FIL & 79/80 & 80/80\\
Plain BF3 & 79/80 & 80/80\\\bottomrule
\end{tabular}
\end{table}

For delta BF1, the median of $g/(2\varepsilon)$ is 37.75, so most rankings are insensitive to its numerical error. The reference maxima also fall on 20 different amplitude pairs, with only 3 of the 80 at $(7,7)$, so the agreement does not come from one corner that always wins.
\paragraph{Region-pair selection.}
In the region-pair study, the decision is which region pair of an image interacts most strongly in the red probability of the five-group distribution. A pair's score is $\max_{i,j>0}|I_{ij,\mathrm{red}}|$ over its 49 amplitudes, and exact ties are broken by lexicographic region-pair order and then by row-major amplitude order. All 11 methods (five arithmetic methods, each with pair-specific and common precision ranges, plus history dense) choose the reference pair and its reference amplitude on all four images. Selecting directly among all $6\times49=294$ mixed points also has zero regret, and delta BF1 already satisfies $g>2\varepsilon$ on every image.

An offline coarse-to-fine simulation also has zero regret. It scores each pair on the $3\times3$ grid formed by amplitudes $\{2/7,4/7,1\}$ on each axis, retains the two highest-scoring pairs, and selects the winner from their full grids. All values come from previously evaluated surfaces, so this diagnostic supplies no timing for a screening-and-refinement execution.

\FloatBarrier
\section{Discussion and Limitations}
\label{sec:validation}
The comparison with the matched control shows a benefit from combining factored operands, identical axis contributions, and higher axis precision. Relative to the control, FIL also changes the BF1 joint operand from $A_i+B_j+C_{ij}$ to $C_{ij}$. A smaller operand can lower the absolute projection error, although $\norm{C_{ij}}$ need not be smaller than $\norm{A_i+B_j+C_{ij}}$ at every input. Separating error coupling from operand scale would require a comparison that changes the axis contributions used by the joints while keeping the component operands fixed. The second-moment decomposition in Section~\ref{sec:ablations} accounts for the aggregate error without identifying these causal contributions. Other anchors, such as the mean candidate input, and precision allocation by token position remain untested.

The evidence covers one approximated MLP in one architecture, development inputs, and five grouped next-token events. Transfer to independent images, other donor colors, multiple layers, and another model remains to be evaluated. Interactions close to the FP32 reference's error level also require a higher-precision reference for the whole nonlinear suffix.

Practical value depends on whether the extra fidelity changes a decision or reduces the cost of meeting a specified tolerance. Both selection tasks are saturated under BF1. A more discriminating evaluation should fix the output event, candidates, acceptable regret, and total compute budget before inspecting method-specific outcomes. Equally optimized full BF3, AtP*, and HVP-based estimators~\cite{kramar2024,zhangwang2026} are relevant comparisons. For a decision needing only four endpoints, the baseline should evaluate those four states rather than the full grid.

Cross-pair native and KV reuse has been implemented but remains unmeasured on the GH200 used here. This optimization is available to every arithmetic method and is separate from FIL's precision allocation.

\section{Conclusion}
FIL evaluates a finite intervention family through native, axis, and mixed components, projects the axes once at higher precision, and reuses them in every candidate that contains them. The local analysis shows which rounding errors then cancel and which remain. In one MLP of a pretrained VLM, FIL reduced the interaction error relative to an axes-only control with the same nominal budget on all 16 synthetic families and all eight main natural images in the two separate studies, including a region-pair comparison in which both methods have identical native and single-axis outputs. Full BF3 remained more accurate in aggregate at similar total query time. The next step is to find an intervention decision that needs this fidelity and to test whether FIL reaches it at lower total cost.

\section*{AI use statement}
GPT-6 Astra was used primarily for experimental implementation. It also supported refinement of hypotheses, formulations and experimental designs, assisted with mathematical proofs and helped review analyses and interpretations.
Further uses included literature search, procedural chart generation and screening, translation of author-written text, and editing and polishing the author-written manuscript.
The author originated the core research idea and initial conceptual and mathematical formulations, and led hypothesis development and the analysis and interpretation of results.

Claude provided editorial reviews and draft revisions, portions of which were incorporated into the manuscript.

Figures were produced programmatically from saved numerical results and chart specifications, without generative image synthesis.

Verification included code tests, replay of saved outputs, and checks against original references.
The author takes responsibility for the final manuscript and artifacts.

\clearpage
\begingroup
\small
\raggedright
\setlength{\bibsep}{2pt}
\bibliographystyle{unsrtnat}
\bibliography{references}
\endgroup

\clearpage
\appendix
\section{Proof of the Nonlinear Bound}
\label{app:proof}
For clarity, consider one pair of nonzero intervention levels. Write $x_{11}=z+a+b+c$, $x_{10}=z+a$, and $x_{01}=z+b$. Define
\begin{align*}
 J_{11}(s)&=Dg(x_{11}+s(e_a+e_b+e_c)),\\
 J_{10}(s)&=Dg(x_{10}+se_a),\\
 J_{01}(s)&=Dg(x_{01}+se_b).
\end{align*}
By the fundamental theorem of calculus, the mixed error is
\begin{align}
 E_g=\int_0^1\big[&(J_{11}(s)-J_{10}(s))e_a\nonumber\\
 &+(J_{11}(s)-J_{01}(s))e_b+J_{11}(s)e_c\big]\,ds.
\end{align}
The respective evaluation-point separations are $b+c+s(e_b+e_c)$ and $a+c+s(e_a+e_c)$. The Lipschitz-Jacobian assumption therefore gives
\begin{align*}
 \norm{J_{11}(s)-J_{10}(s)}&\le L_2(\norm{b+c}+s(\epsilon_b+\epsilon_c)),\\
 \norm{J_{11}(s)-J_{01}(s)}&\le L_2(\norm{a+c}+s(\epsilon_a+\epsilon_c)).
\end{align*}
Applying the operator-norm bounds, integrating $s$ over $[0,1]$, and using $\norm{J_{11}}\le L_1$ yields Eq.~\eqref{eq:bound}. Substitution of the stated amplitude assumptions proves the corollary. The result also applies to a local SwiGLU map by treating the gate/up pair as one vector, on a bounded region with finite derivative constants. It does not bound every downstream attention layer with the same constants.

\paragraph{Rounding model.}
For normal finite operands without range failure, a componentwise BF16 cast satisfies a relative error bound with unit roundoff $2^{-8}$. A local matrix-product error additionally depends on weight rounding and accumulation. Even with an exact BF16 weight, a bound must include accumulation, not just input casting. In the CPU probe, the linear diagnostic found no violation of its input-rounding bound in the tested cases; this is an empirical check, not a bound for arbitrary GEMM reduction orders.

The measured defect between the reconstructed mixed projection and its directly projected mixed component is $1.11\times10^{-16}$ in FP64 and $1.64\times10^{-7}$ in FP32 in the CPU diagnostic. These measured defects quantify reconstruction error in this CPU diagnostic.

\section{Implementation Requirements}
\label{app:contract}
The single-MLP experiments follow these requirements, which would also apply to a multilayer extension (not yet evaluated).
\begin{enumerate}
\item Bind each gate/up/down operation to the parameters of its own layer. Preserve the original positions, causal masks, normalization, attention, and residual topology.
\item Form $H_{00},H_{i0},H_{0j},H_{ij}$ at the same layer, token position, projection, and execution stage. Decompose after the relevant normalization.
\item Use identical numerical axis projections in corresponding candidates; caching avoids repeated work. Deterministic recomputation satisfies the same fidelity condition only when numerical identity is verified. Do not reset a joint candidate to an axis or reference state, and keep each joint candidate's own KV history once histories diverge.
\item Keep all 49 mixed residuals and full feature dimensions. Approximate only the selected gate/up projections, and only after the full-branch frontier.
\item Keep range checks, fallbacks, candidate-order tests, and replay checks against the native and full model. An accuracy failure is a result to report; a violated requirement invalidates the run.
\item Include source acquisition, history construction, native projections, casts, mixed subtraction, reconstruction, temporary memory, and any fallback in runtime accounting. Report the memory occupied by the precomputed BF16 high and residual copies of the gate/up weights separately, along with the time for preparing them and for first-use kernel compilation.
\end{enumerate}
When candidate microbatching changes, axis projections must still be available to every associated joint with compatible token geometry. Recomputing them with a different batch shape may change rounding. The same-shape recomputation control in Section~\ref{sec:allocation} verifies numerical identity and measures the cost of recomputation on GH200; it does not cover arbitrary shape changes or a multilayer schedule.

\section{Component-Formation Diagnostic}
\label{app:formation}
At each fixed reached input, the diagnostic compares the reconstructed projection's four-term difference with $WC^{(64)}$, where the mixed component and the original-weight product are evaluated in FP64. The denominator called direct SSE is $\|\Dmix\widehat Z-WC^{(64)}\|_F^2$, summed over the 49 mixed candidates, active tokens, and projection features, with the four-term difference evaluated in FP64. It is a local projection error, not the final VLM output error.

Four families, two projections, three subtraction orders, and two precisions each for formation and reconstruction give 96 configurations. Under sequential FP32 formation and reconstruction, the median formation-SSE/direct-SSE ratios are $1.71\times10^{-8}$ for gate and $1.60\times10^{-8}$ for up; the reconstruction ratios are $1.10\times10^{-6}$ and $4.36\times10^{-7}$. Formation measures $W(\widetilde C-C^{(64)})$, and reconstruction measures the reconstructed four-term difference minus the computed mixed projection. The diagnostic also separates the FP32 operand cast from the BF1 mixed-projection error. Signed cross terms prevent these ratios from being interpreted as additive shares of total SSE.

\section{Data, Aggregation, and Availability}
\label{app:provenance}
Tables and graphs are computed from saved prediction summaries, timing repetitions, and diagnostics. The VLM losses use five output groups, including OTHER; the CPU probe uses four softmax outputs. The reported aggregates include all evaluated families and timing repetitions. Image 61, which is used only to recheck the baseline failure case, is kept out of the natural-image averages by design.

Synthetic descriptive intervals use 50,000 paired-family bootstrap draws with seed 145015. Natural-image baseline intervals preserve the image cluster. In the two-question FIL study, intervals enumerate all 256 image-cluster resamples. These intervals describe the evaluated development sets and do not correct for repeated model-development decisions or multiple comparisons.

The analyses include rescoring of saved outputs and checks that implementation variants give identical predictions. The accompanying archive contains the LaTeX source and vector figures. Experimental code and raw result archives are not included in this release.

\section{Packing, Preparation, and Storage}
\label{sec:storage-reuse}
The storage comparison uses the same 16 synthetic families and single MLP, with timing on four families. Here ``packed'' denotes the common-operand implementation; ``storage'' adds temporary-storage reuse to that implementation. Both use the same arithmetic and skip the zero weight-residual product. Across the six storage/packed pairs, the saved logits agree for all 96 method--family comparisons. All 208 accuracy records pass the fixed five-group maximum-error $0.01$ and endpoint-MSE $10^{-4}$ gates, with no recorded arithmetic fallback. 

\begin{table}[!htbp]
\centering\small
\caption{Packing/storage comparison: 16-family accuracy and four-family timing. Times average seven-repeat family medians, including source cost. All methods use the audited zero-product optimization.}
\label{tab:storage-reuse}
\begin{tabular}{lrr}\toprule
Method & Mean interaction MSE & Primary query time (ms)\\\midrule
History dense & $1.77\times10^{-12}$ & 1634.862\\
Delta BF1 packed & $7.97\times10^{-10}$ & 1627.608\\
Delta BF1 storage & $7.97\times10^{-10}$ & 1627.512\\
Axes-only BF3 packed & $5.02\times10^{-10}$ & 1630.131\\
Axes-only BF3 storage & $5.02\times10^{-10}$ & 1629.602\\
Factored BF1 packed & $9.72\times10^{-10}$ & 1631.247\\
Factored BF1 storage & $9.72\times10^{-10}$ & 1630.208\\
FIL packed & $2.81\times10^{-11}$ & 1631.797\\
FIL storage & $2.81\times10^{-11}$ & 1631.364\\
Delta BF3 packed & $3.53\times10^{-12}$ & 1631.064\\
Delta BF3 storage & $3.53\times10^{-12}$ & 1630.977\\
Plain BF3 packed & $2.13\times10^{-12}$ & 1630.324\\
Plain BF3 storage & $2.13\times10^{-12}$ & 1630.461\\
\bottomrule\end{tabular}
\end{table}

FIL storage keeps the 94.40\% MSE reduction over axes-only BF3 storage and improves all 16 families, but its primary query time is 1.762\,ms longer. Its MSE is 13.16 times that of plain BF3 packed, which is 1.040\,ms faster, and 7.95 times that of delta BF3 storage, which is 0.388\,ms faster. In this run, the nominal saving from 126 to 77 row-products therefore does not give FIL a better accuracy--time trade-off than full BF3. We compare with plain BF3 packed because it is faster than plain BF3 storage at the same accuracy (Figure~\ref{fig:storage-reuse_cost}a).

\paragraph{Timing statistic.}
Relative to FIL packed, storage lowers the primary statistic by 0.433\,ms, with a decrease in all four family medians. The arithmetic mean of the same 28 repetitions per variant instead rises by 3.378\,ms (Table~\ref{tab:storage-reuse-timing}, Figure~\ref{fig:storage-reuse_cost}b). Both summaries use the same measurements; only the first was specified in advance as the primary statistic.

\begin{table}[!htbp]
\centering\small
\caption{FIL timing sensitivity in the same storage-reuse run. No repetition is excluded. All entries are in milliseconds.}
\label{tab:storage-reuse-timing}
\begin{tabular}{lrrr}\toprule
Statistic & Packed & Storage & Storage $-$ packed\\\midrule
Mean of four family medians & 1631.797 & 1631.364 & $-0.433$\\
Mean of all 28 repetitions & 1631.709 & 1635.087 & $+3.378$\\\bottomrule
\end{tabular}
\end{table}

One long repetition strongly affects the arithmetic mean: repetition 6 of FIL storage in family 214 takes 1843.2\,ms, and its native stage takes 193.8\,ms against 95.8--98.4\,ms in the other six repetitions. The delay occurs in a stage that storage reuse does not change, and no repetition is excluded. Across all 56 timing summaries, the interquartile range spans 0.24--11.78\,ms. Source acquisition is measured once per family and added to each repetition, so its cost is included but its variability is not. Four families with seven repetitions each are too few for tail-latency claims or for speedups on new images.

\begin{figure}[!htbp]
\centering
\includegraphics[width=\linewidth]{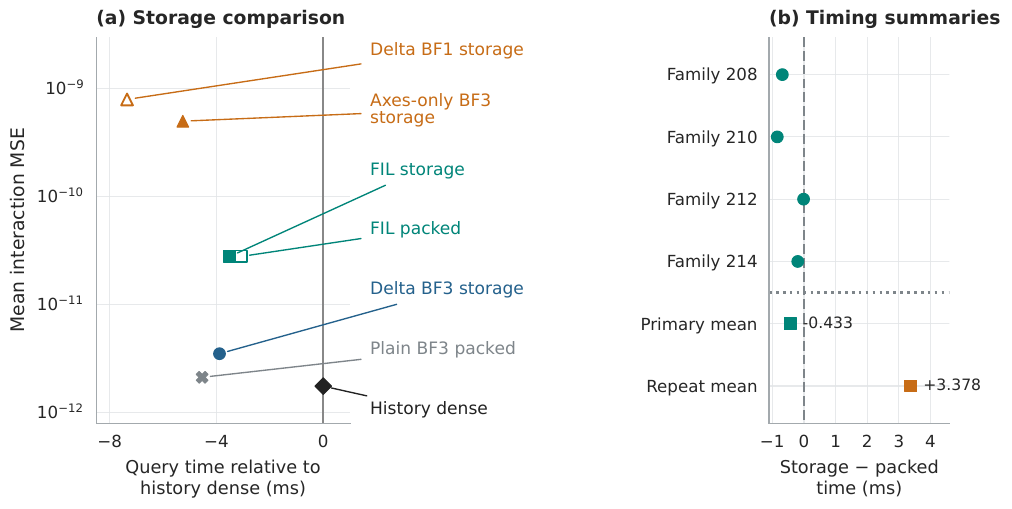}
\caption{Storage-reuse comparison. (a) Selected methods from Table~\ref{tab:storage-reuse}: 16-family accuracy and four-family timing, relative to history dense (1634.862\,ms). (b) Storage minus packed time: per-family differences of medians, their mean, and the difference of the arithmetic means of all 28 repetitions, which has the opposite sign.}
\label{fig:storage-reuse_cost}
\end{figure}

\paragraph{Packing and preparation sharing.}
A modular version of the packed code in Table~\ref{tab:storage-reuse} independently toggles packing ($P$) and shared gate/up component preparation and cast checks ($S$), with temporary-storage reuse disabled (Table~\ref{tab:factorial}). Preparation sharing has the larger effect. The unpacked FIL and delta BF3 projection paths from Table~\ref{tab:synthetic-main} are retimed as references in the same local benchmark: $p_0s_0$ is 1.038 and 1.095\,ms slower, respectively. The contrasts isolate choices within the four-arm implementation, rather than decomposing the whole-query speedup.

\begin{table}[!htbp]
\centering\small
\caption{Local-MLP packing/preparation $2\times2$ ablation. Cell means and paired differences average four families, in milliseconds. The difference of differences is $(p_1s_1-p_0s_1)-(p_1s_0-p_0s_0)$. Contrasts are computed before rounding.}
\label{tab:factorial}
\begin{tabular}{lrr}\toprule
Condition or contrast & FIL & Delta BF3\\\midrule
No packing, separate preparation: $p_0s_0$ & 24.83 & 23.45\\
No packing, shared preparation: $p_0s_1$ & 22.17 & 20.80\\
Packing, separate preparation: $p_1s_0$ & 24.14 & 23.14\\
Packing, shared preparation: $p_1s_1$ & 21.49 & 20.50\\
Table~\ref{tab:synthetic-main} projection path, local timing & 23.79 & 22.35\\\midrule
Packing with shared preparation: $p_1s_1-p_0s_1$ & $-0.682$ & $-0.298$\\
Preparation sharing with packing: $p_1s_1-p_1s_0$ & $-2.653$ & $-2.643$\\
Difference of differences & $+0.009$ & $+0.007$\\\bottomrule
\end{tabular}
\end{table}

For each gate or up projection, the right operand has shape $2560\times9728$. With $T$ active tokens, unpacked FIL uses two left operands of shape $14T\times2560$ (axis high and residual components) and one of shape $49T\times2560$ (mixed high component). Packing concatenates these into one $77T\times2560$ operand. Delta BF3 instead uses two $63T\times2560$ operands, or one packed $126T\times2560$ operand. The timed families 208, 210, 212, and 214 have $T=178,178,94,178$, respectively. Preparation sharing forms and casts the input components and checks their cast safety once for gate and up; separate preparation repeats those operations for each projection. At fixed packing, sharing leaves the GEMM shapes unchanged.

\paragraph{Memory.}
In the local FIL MLP, storage reuse lowers the mean time by 0.499\,ms. The incremental peak allocation falls from 2581.16 to 2243.50\,MiB for long inputs and from 1363.77 to 1185.46\,MiB for the short input, but the whole-query peak is 22133.33\,MiB in both variants. Profiler traces suggest that projection savings and reconstruction costs compete, but overlapping intervals and transfer counts do not show that index conversion is the cause. 

\FloatBarrier
\section{Audit Batching and Numerical Equivalence}
\label{sec:audit-batching}
The audit comparison in Table~\ref{tab:audit-batching} uses the same four synthetic families at decoder 13. It changes retrieval of device audit results after fixed index reuse and reconstruction, preserving integer counters, finite checks, and fallback decisions. Matched original/batched variants have identical saved predictions. All conditions pass the fixed fidelity gates, with zero arithmetic fallbacks.

\begin{table}[!htbp]
\centering\small
\caption{Batched minus original audit retrieval within the audit-batching comparison. Synchronization counts are mean profiler counts over four families. Local and whole-query time changes are in milliseconds, measured separately.}
\label{tab:audit}
\begin{tabular}{lrrr}\toprule
Method & Synchronizations, original $\to$ batched & Local MLP $\Delta$ & Whole-query $\Delta$\\\midrule
FIL & $24\to12$ & $-0.542$ & $+0.080$\\
Axes-only BF3 & $24\to12$ & $-0.366$ & $-0.322$\\
Delta BF3 & $15\to8$ & $-0.215$ & $-0.888$\\
Plain BF3 & $15\to8$ & $-0.214$ & $+0.113$\\\bottomrule
\end{tabular}
\end{table}

FIL's local MLP time falls from 19.90 to 19.35\,ms while its whole-query time rises by 0.080\,ms. The two scopes were measured separately and cannot be added into an exact timing decomposition. Execution context, cache state, allocation, and temporal variation are possible sources of the difference, but our records do not separate them. FIL's whole-query peak memory is about 22133\,MiB and its local incremental peak about 1979\,MiB, both unchanged by batching. The precomputed BF16 high and residual copies of the gate/up weights occupy 190\,MiB in addition to the FP32 parameters. The zero residual copies remain allocated even though their products are skipped.

In the natural-image run (Table~\ref{tab:natural-original-audit}), the original and batched variants give identical logits in all 40 paired method--case comparisons. Batching lowers FIL's local time from 15.34 to 14.95\,ms and its synchronizations from 24 to 12, while the whole-query time falls by only 0.06\,ms (image-cluster 95\% interval $[-0.54,+0.63]$\,ms).

\begin{table}[!htbp]
\centering\small
\caption{Natural-image run with both audit read-back variants: four image means, each averaging two questions, and the mean over images of seven-repeat case medians. Matched original and batched variants give identical logits.}
\label{tab:natural-original-audit}
\begin{tabular}{lrr}\toprule
Method & Mean interaction MSE & Primary query time (ms)\\\midrule
History dense & $6.96\times10^{-12}$ & 1405.438\\
Delta BF1 & $8.34\times10^{-9}$ & 1398.411\\
Axes-only BF3, original audit & $7.72\times10^{-9}$ & 1399.489\\
FIL, original audit & $1.16\times10^{-11}$ & 1400.179\\
Delta BF3, original audit & $7.71\times10^{-12}$ & 1400.362\\
Plain BF3, original audit & $1.08\times10^{-11}$ & 1400.378\\
Axes-only BF3, batched audit & $7.72\times10^{-9}$ & 1399.353\\
FIL, batched audit & $1.16\times10^{-11}$ & 1400.118\\
Delta BF3, batched audit & $7.71\times10^{-12}$ & 1400.542\\
Plain BF3, batched audit & $1.08\times10^{-11}$ & 1400.070\\
\bottomrule\end{tabular}
\end{table}

Applying the decomposition of Eq.~\eqref{eq:moment-split} to the four-family subset, the diagonal and signed cross changes account for 99.88\% and 0.12\% of the loss reduction. On the natural images, the diagonal and signed cross sums are $7.61\times10^{-9}$ and $1.10\times10^{-10}$ for axes-only BF3 and $1.38\times10^{-11}$ and $-2.16\times10^{-12}$ for FIL.

\section{Four-MLP Natural-Image Baselines}
\label{sec:natural-baseline}
This baseline study evaluates six existing precision methods on 64 COCO val2014 images~\cite{lin2014}, two questions per image, and four MLPs (13--16). It contains no FIL condition. Both questions and all candidates stay within their image cluster. Appendix~\ref{app:natural} specifies the preprocessing and intervention rules, including which questions have not been reviewed for semantic ambiguity.

\begin{table}[t]
\centering\small
\caption{Four-MLP natural-image baseline evaluation: 128 cases in 64 image clusters, four MLPs. The two questions per image receive equal weight. Interaction MSE is the image-mean five-group loss; Max error is the global maximum probability error over cases, states, and groups; $^\dagger$ exceeds the 0.01 gate. Time differences use within-case seven-repeat medians on eight images with two questions each, then average by image. FIL is not evaluated here.}
\label{tab:natural-baseline}
\begin{tabular}{lrrr}\toprule
Method & Interaction MSE & Max error & Time $-$ history dense (ms)\\\midrule
History dense & $7.29\times10^{-12}$ & $3.16\times10^{-5}$ & 0.000\\
Plain BF1 & $2.06\times10^{-7}$ & $8.41\times10^{-3}$ & $-32.843$\\
Delta BF1 & $2.53\times10^{-8}$ & $1.32\times10^{-2}$$^\dagger$ & $-31.324$\\
Plain BF3 & $9.42\times10^{-12}$ & $6.98\times10^{-5}$ & $-14.656$\\
Delta BF3 & $6.88\times10^{-12}$ & $3.62\times10^{-5}$ & $-12.912$\\
Split8 BF1 & $2.56\times10^{-11}$ & $9.21\times10^{-5}$ & $-11.245$\\\bottomrule
\end{tabular}
\end{table}

Delta BF1 reduces image-mean interaction MSE by $87.69\%$ relative to plain BF1 and improves 63/64 images. The paired mean difference is $-1.81\times10^{-7}$, with a descriptive image-bootstrap 95\% interval $[-2.99\times10^{-7},-8.76\times10^{-8}]$. On the timing subset it is $1.518$\,ms slower on average and slower in all eight image clusters. 

Figure~\ref{fig:natural_tail} shows all 64 paired image means and all 128 case-level maximum errors.

\begin{figure}[!htbp]
\centering
\includegraphics[width=\linewidth]{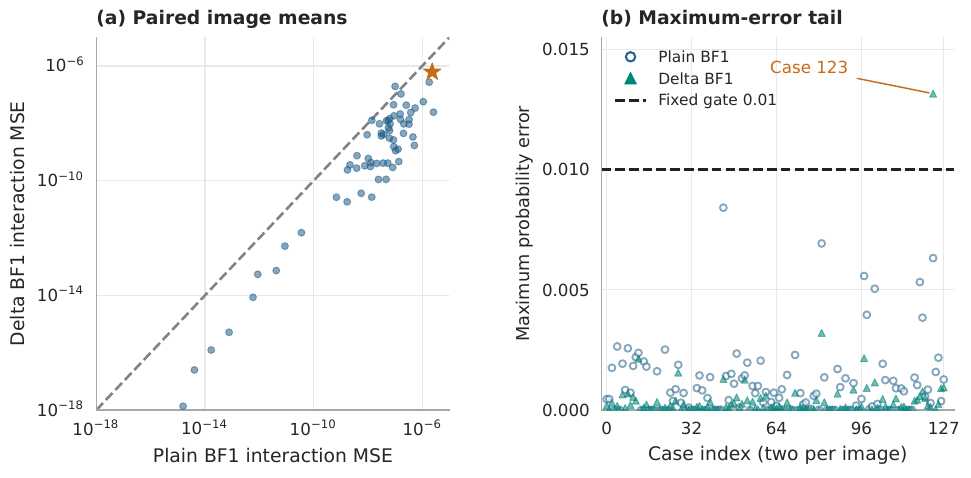}
\caption{Four-MLP natural-image baselines, without FIL. (a) Image-mean paired MSE over 64 images. (b) Maximum error for all 128 cases. The star marks the parent image of case 123, which exceeds the fixed 0.01 gate.}
\label{fig:natural_tail}
\end{figure}

\paragraph{A maximum-error failure.}
Case 123 (image 61, COCO image 108669) has maximum probability error 0.0132 under delta BF1, above the 0.01 gate, although its endpoint MSE of $9.28\times10^{-7}$ passes the $10^{-4}$ gate. This one case contributes 57.26\% of delta BF1's PairLoss sum and 37.95\% of its interaction-MSE sum, and it is included in every aggregate.

The largest discrepancy is at candidate 41 ($i=5$, $j=1$) in OTHER, where the reference probability 0.456354 becomes 0.469509 under delta BF1; the total four-color mass changes by the same amount in the opposite direction. After renormalizing over the four colors, the maximum error over all 128 cases is only 0.00298, but that is a different output event. OTHER covers the whole vocabulary outside the four color groups; these grouped outputs do not identify which non-color tokens cause the discrepancy.

Figure~\ref{fig:case123} shows the signed five-group discrepancy at the worst candidate.

\begin{figure}[!htbp]
\centering
\includegraphics[width=.72\linewidth]{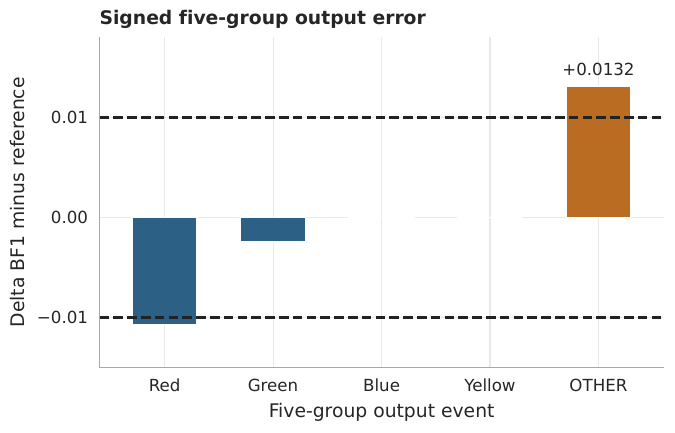}
\caption{Signed five-group probability errors for case 123 at $i=5,j=1$, with the fixed $\pm0.01$ gate. The OTHER error exceeds the gate. The case remains in all 128-case aggregates.}
\label{fig:case123}
\end{figure}

Image 61 is a monochrome photograph with several apples, while the question asks only about ``the apple.'' The ambiguous reference limits semantic interpretation. Its connection to the numerical error is unknown, so the case remains in the evaluation.

\FloatBarrier
\paragraph{Complete image inventory.}
Table~\ref{tab:coco-inventory} identifies all 64 images and their two source annotations. Both questions are timed for image indices 0, 8, 16, 24, 32, 40, 48, and 56. The annotation categories supply the question words, and the ordered A/B masks supply the donors under the rules in Appendix~\ref{app:natural}.

\begin{table}[!htbp]\centering\footnotesize
\caption{Complete COCO val2014 inventory for the four-MLP baseline study. A and B are annotation IDs. Indices are local to this inventory; the region-pair study uses the different images in Table~\ref{tab:region-stimuli}.}
\label{tab:coco-inventory}
\begin{tabular}{rrrr@{\hspace{1.2em}}rrrr}\toprule
Index & Image ID & A & B & Index & Image ID & A & B\\\midrule
0 & 274083 & 140224 & 168652 & 32 & 302219 & 1165906 & 2211551\\
1 & 137004 & 1538899 & 1558407 & 33 & 412136 & 78954 & 718171\\
2 & 346752 & 1531577 & 1913818 & 34 & 382207 & 90480 & 1537396\\
3 & 241326 & 9113 & 47558 & 35 & 70657 & 1041414 & 2109880\\
4 & 524775 & 138571 & 168061 & 36 & 375461 & 1046044 & 1536142\\
5 & 48014 & 11478 & 137103 & 37 & 197915 & 38171 & 82048\\
6 & 56714 & 1041988 & 2111352 & 38 & 460783 & 53143 & 1042285\\
7 & 323515 & 1513415 & 1537112 & 39 & 549518 & 713545 & 1043471\\
8 & 494077 & 1041753 & 1051429 & 40 & 272615 & 714126 & 1042498\\
9 & 311041 & 1051137 & 1506415 & 41 & 267182 & 1905054 & 1906261\\
10 & 335810 & 1556588 & 1904525 & 42 & 94563 & 1049475 & 1054164\\
11 & 484080 & 679708 & 1556090 & 43 & 171199 & 92376 & 672757\\
12 & 200365 & 125958 & 133079 & 44 & 238498 & 86154 & 1901875\\
13 & 371890 & 48948 & 125346 & 45 & 455345 & 1538356 & 1548380\\
14 & 22446 & 679999 & 1039246 & 46 & 307138 & 1556515 & 2093456\\
15 & 548506 & 1043772 & 1047747 & 47 & 577451 & 1049391 & 1053414\\
16 & 64307 & 676954 & 710873 & 48 & 59652 & 6832 & 23168\\
17 & 511111 & 678816 & 1053251 & 49 & 42215 & 1052019 & 1509788\\
18 & 286322 & 88332 & 714742 & 50 & 128598 & 1050705 & 1906607\\
19 & 54598 & 671772 & 1171738 & 51 & 221094 & 137822 & 1955661\\
20 & 368096 & 709798 & 1052165 & 52 & 294297 & 716972 & 1864712\\
21 & 131416 & 1607627 & 1883690 & 53 & 37688 & 674039 & 2107624\\
22 & 250790 & 710099 & 1882909 & 54 & 343248 & 81232 & 2213786\\
23 & 335148 & 1178926 & 1531897 & 55 & 535013 & 1043367 & 1879310\\
24 & 282790 & 1835086 & 2073357 & 56 & 79651 & 83672 & 1041806\\
25 & 495419 & 86544 & 675993 & 57 & 177449 & 1050770 & 2111859\\
26 & 316553 & 3903 & 60885 & 58 & 397719 & 89305 & 682175\\
27 & 570548 & 1041778 & 1900441 & 59 & 464622 & 1514191 & 2094395\\
28 & 409042 & 21372 & 709846 & 60 & 313718 & 55283 & 168669\\
29 & 450728 & 2037088 & 2050386 & 61 & 108669 & 1535648 & 1548458\\
30 & 298138 & 92669 & 1507425 & 62 & 494580 & 82432 & 1537498\\
31 & 301014 & 50858 & 712063 & 63 & 190156 & 52155 & 672670\\
\bottomrule\end{tabular}\end{table}

\section{Native-Relative and Low-Rank Baselines}
\label{app:historical}
These earlier studies motivated the execution design but do not evaluate FIL. Their four-MLP timings come from a different setup and are not comparable with the single-MLP timings in Section~\ref{sec:experiments}.

\subsection{Native-relative computation on the frozen VLM}
The native-relative baseline experiment uses Qwen3-VL-4B-Instruct~\cite{qwen2025} with four selected decoder MLPs, indices 13--16, hidden width 2560, and MLP width 9728. Each synthetic family has an $8\times8$ grid with strengths from 0 to 1. The comparison was fixed for 64 new synthetic families after development on the source indices 208--223 of Appendix~\ref{app:synthetic}. The 64 new families use indices $i=0,\ldots,63$, scene identifiers $160000000+i$, and PCG64 seeds $191601016+i$, with the same rendering, draw sequence, question, and donor rules as Appendix~\ref{app:synthetic}. Before model evaluation, we rejected a family if any native or donor scene repeated a scene within that family, among selected families, or in the prior development inventory. Scene identity uses RGB fills, layout, appearance, and checkerboard phase. The first 64 draws passed this check. Timing uses indices 0, 8, 16, 24, 32, 40, 48, and 56. For the present paper, these 64 families are development data.

The implementation keeps native computation, attention, nonlinearities, full down projections, and residual updates in the original FP32 path. Selected gate/up operands use BF16 with FP32 accumulation; TF32 and reduced-precision BF16 reductions are disabled. The backend is an NVIDIA GH200 120GB with PyTorch 2.11.0+cu130 and Transformers 5.12.1. The model/tokenizer configuration and output-token definitions are given in Appendix~\ref{app:tokens}. Timing uses eight fixed families, two warmups, and seven repetitions per method; reported times average the within-family medians. Range checks and FP32 fallbacks were active, and no arithmetic fallback occurred.

\begin{table}[t]
\centering\small
\caption{Native-relative baselines: 64 synthetic families, four MLPs. Losses are family means; Max error is the global five-group maximum. Time averages seven-repeat medians for eight families.}
\label{tab:vlm}
\begin{tabular}{lrrrrr}
\toprule
Method & PairLoss $\downarrow$ & Interaction MSE $\downarrow$ & Max error $\downarrow$ & Total ms & $\Delta$ ms vs. history dense\\
\midrule
History dense & $3.73\!\times\!10^{-12}$ & $3.04\!\times\!10^{-12}$ & $3.032\!\times\!10^{-5}$ & 1406.097 & ---\\
Plain BF1 & $4.17\!\times\!10^{-8}$ & $3.32\!\times\!10^{-8}$ & $5.296\!\times\!10^{-3}$ & 1384.171 & $-21.926$\\
Delta BF1 & $4.88\!\times\!10^{-9}$ & $3.71\!\times\!10^{-9}$ & $2.030\!\times\!10^{-3}$ & 1385.366 & $-20.731$\\
Plain BF3 & $4.24\!\times\!10^{-12}$ & $3.74\!\times\!10^{-12}$ & $3.131\!\times\!10^{-5}$ & 1400.289 & $-5.808$\\
Delta BF3 & $3.41\!\times\!10^{-12}$ & $3.27\!\times\!10^{-12}$ & $3.105\!\times\!10^{-5}$ & 1401.352 & $-4.745$\\
Split8 BF1 & $1.96\!\times\!10^{-11}$ & $2.26\!\times\!10^{-11}$ & $8.596\!\times\!10^{-5}$ & 1403.517 & $-2.581$\\
Input cap40 & $1.11\!\times\!10^{-8}$ & $8.53\!\times\!10^{-9}$ & $3.201\!\times\!10^{-3}$ & 1414.012 & $+7.915$\\
\bottomrule
\end{tabular}
\end{table}

Delta BF1 reduces the mean PairLoss by 88.32\% relative to plain BF1 and improves 62 of 64 families (Table~\ref{tab:vlm}). Its global maximum probability error falls from $5.30\times10^{-3}$ to $2.03\times10^{-3}$. It costs 1.20\,ms more on average and is slower in all eight timed families, so it buys accuracy at a small cost. 

Relative to history dense, delta BF1 saves 20.731\,ms (1.47\% of the mean time); per-family speedups range from 1.003 to 1.025, and five of eight families pass a $1.01\times$ speedup threshold fixed in advance. Ordinary unshared execution takes 2058.096\,ms on average in the same data, so most of its difference from the shared implementations comes from history sharing rather than native centering. Split8 BF1 is slower than plain BF3 in all eight timed families and has a higher mean PairLoss.

\paragraph{Candidate-subspace baselines.}
Split8 BF1 and Input cap40 use a proper orthogonal decomposition across the 63 non-native candidates. At each selected MLP, let the $63\times T\times d$ tensor $\mathcal D$ contain the current post-normalization inputs minus the native input. Flatten only its token and feature dimensions to obtain the $63\times Td$ matrix $D$. With $q_0=\mathbf1/\sqrt{63}$ and $J=I-q_0q_0^\top$, the basis $Q_r$ contains $q_0$ and up to $r-1$ leading positive eigenvectors of $JDD^\top J$. Thus the rank refers to the candidate dimension, with all active tokens and features pooled in one Gram matrix. The basis is recomputed at each MLP and shared by gate and up.

We accumulate the Gram matrix in FP64 over blocks of 131072 flattened columns. An eigenvalue is treated as positive only if it exceeds $128u_{64}\max(|\operatorname{tr}(DD^\top)|,\lambda_1,10^{-300})$, where $u_{64}$ is the FP64 machine epsilon and $\lambda_1$ is the largest centered eigenvalue. When the rank cap would split a cluster of eigenvalues whose consecutive gaps are at most $10^{-10}\lambda_1$, the whole cluster is dropped. Both rank caps include the constant direction. Dropping a boundary cluster or having fewer positive eigenvalues can reduce the effective rank below 8 or 40. The centered eigenvectors are reorthogonalized in FP64, and candidate mixing uses FP32 copies of the basis.

Write $\Pi_W(M)$ for applying the FP32 linear map $W$ to every token of tensor $M$, and $\Pi_{1,W}(M)$ for its BF1 counterpart. Candidate mixing acts on the first tensor dimension. With $M_r=Q_r^\top\mathcal D$ and $R_r=\mathcal D-Q_rM_r$, the gate/up outputs are
\[
\begin{aligned}
\widehat Z_{\mathrm{Split8}}&=Z_0+Q_8\Pi_W(M_8)+\Pi_{1,W}(R_8),\\
\widehat Z_{\mathrm{cap40}}&=Z_0+Q_{40}\Pi_W(M_{40}).
\end{aligned}
\]
The native output $Z_0$ is broadcast across candidates. Split8 retains the full residual at BF1; cap40 discards the residual. Formation, mixing, and reconstruction use FP32. Both evaluate SiLU, gate--up multiplication, the full down projection, and residual updates for every candidate; neither truncates the nonlinear output to the candidate subspace.

\paragraph{Why the existing gate is insufficient.}
All seven methods pass the original gates (maximum probability error $\le0.01$, endpoint MSE $\le10^{-4}$) on all 64 families. Each also has zero sign errors on the same 2,233 interaction components with reference magnitude above $10^{-4}$. These components are clustered within the 64 families. The gates therefore do not separate the methods.

A post hoc tolerance curve does (Figure~\ref{fig:tolerance}). At maximum-error thresholds $10^{-4}$, $5\times10^{-4}$, and $10^{-3}$, plain BF1 passes 8, 22, and 44 families; delta BF1 passes 24, 49, and 61. Both BF3 variants pass all 64 at each of these thresholds. These stricter thresholds are descriptive and were not preregistered.

\begin{figure}[!htbp]
\centering
\includegraphics[width=\linewidth]{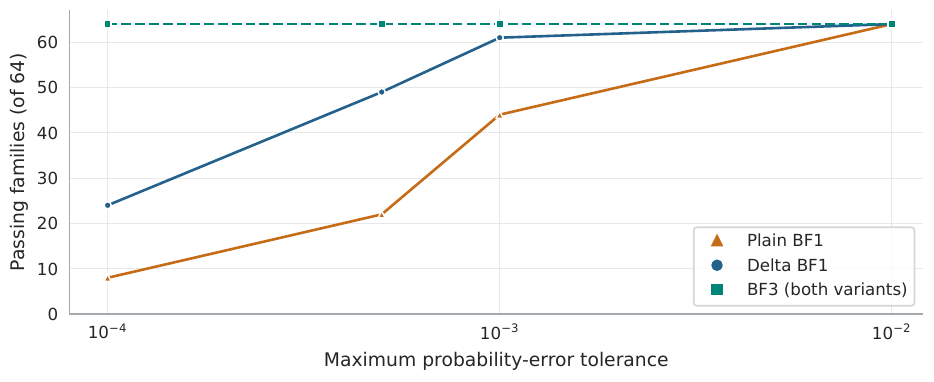}
\caption{Post hoc tolerance curve for the 64-family native-relative study. Both BF3 variants pass all 64 families at every threshold shown; the original gate is 0.01.}\label{fig:tolerance}\end{figure}

\subsection{Limits of a small static error subspace}
The static low-rank correction experiment applies a correction to delta BF1 at decoder MLP 13, using the 16 development families 208--223. The calibration folds are $F_A=\{208,209,212,213,216,217,220,221\}$ and $F_B=\{210,211,214,215,218,219,222,223\}$. Each family is evaluated with factors fitted to the other fold; timing uses 208, 210, 212, and 214. A rank-32 measured projection correction reduces PairLoss from $8.57\times10^{-10}$ to $8.40\times10^{-10}$ (1.99\%) and improves eight of 16 families. The paired-family bootstrap interval for the mean difference, $[-9.41,4.93]\times10^{-11}$, includes zero. Its global maximum probability error increases from 0.000385 to 0.000402, and the average time increases by 3.187\,ms across four timed families.

For gate and up independently, stack the native-relative inputs of all 63 non-native candidates and active tokens of family $f$ into an $n_f\times2560$ matrix $D_f$. Let $Y_f$ be its BF1 projection and $E_f=Y_f-\operatorname{FP32}(D_fW^\top)$. The calibration matrix is the uncentered, equal-family second moment
\[
 M_F=\frac1{8}\sum_{f\in F}\frac{E_f^\top E_f}{n_f}.
\]
Subtraction and moment accumulation use FP64. The leading eigenvectors of $M_F$ give a $9728\times r$ matrix $V$, which is stored in FP32. Using that stored basis, $A=V^\top W$ is evaluated in FP64 and rounded to FP32. As for the subspace baselines, a cluster of eigenvalues that the rank cap would split, with relative gaps at most $10^{-10}$, is dropped in full. At evaluation, the correction is
\[
 Y_{\mathrm{corr}}=Y+(DA^\top-YV)V^\top,
\]
with the thin products and additions in FP32. The native projection is then added, followed by the original nonlinear continuation. In exact arithmetic with orthonormal $V$, the correction removes $EVV^\top$. Error-energy recovery for one family is $\lVert E_fV\rVert_F^2/\lVert E_f\rVert_F^2$, measured against the FP32 linear action before native addition. Table~\ref{tab:rank} takes its median over the 16 opposite-fold evaluations, separately for gate and up.

The opposite-fold median error-energy recovery is low (Table~\ref{tab:rank}). At rank 32, it is $7.09\%$ for gate and $3.70\%$ for up. The measured FP32 correction is close to ideal projection in the saved diagnostics, so poor subtraction accuracy is not the main explanation for the limited recovery. The result limits this particular fixed-subspace correction; methods with other calibration, ranks, or structures require their own evaluation.

\begin{table}[t]
\centering\small
\caption{Static low-rank correction: median error-energy recovery on the opposite calibration fold (16 development families).}
\label{tab:rank}
\begin{tabular}{rrr}\toprule
Rank & Gate & Up\\\midrule
32 & $7.09\%$ & $3.70\%$\\
64 & $10.61\%$ & $6.61\%$\\
128 & $16.14\%$ & $11.71\%$\\\bottomrule
\end{tabular}
\end{table}

\section{Controlled CPU Experiment}
\label{app:reproduce}
This controlled network isolates the amplitude dependence of the local numerical effect.

\subsection{Setup and results}
\label{sec:cpu}
The NumPy probe constructs one random two-layer RMSNorm--SwiGLU--residual network with hidden width 32, MLP width 48, and a four-output softmax. Thirty-two generated families share its weights; they are not 32 independently trained models. All weights are rounded once to exactly representable BF16 values. The probe emulates round-to-nearest-even operand casts and performs matrix multiplication and nonlinear computation in FP64 or FP32. The reference is dense FP64. There is no attention, pretrained checkpoint, GPU kernel, or runtime measurement.

Each family starts at $x_0+t(\alpha_i a+\beta_j b)$. Every selected projection decomposes the actual post-normalization state, including mixed components created by the first layer. Table~\ref{tab:cpu} reports the probe results at $t=1/16$ with FP32 execution. The two matched-precision methods are Factored axes BF3 and axes-only BF3 with independent joints.

\begin{table}[t]
\centering\small
\caption{CPU probe at $t=1/16$: FP32 execution against dense FP64. RMS pools 32 families in one generated network.}
\label{tab:cpu}
\begin{tabular}{lrr}\toprule
Method & Interaction RMS & Endpoint RMS\\\midrule
Delta BF1 & $4.54\!\times\!10^{-6}$ & $2.97\!\times\!10^{-6}$\\
Factored BF1 & $1.48\!\times\!10^{-7}$ & $2.99\!\times\!10^{-6}$\\
Axes-only BF3 & $3.17\!\times\!10^{-6}$ & $2.77\!\times\!10^{-6}$\\
Factored axes BF3 & $5.13\!\times\!10^{-8}$ & $3.91\!\times\!10^{-8}$\\
Delta BF3 & $3.23\!\times\!10^{-8}$ & $1.68\!\times\!10^{-8}$\\
Dense FP32 & $3.26\!\times\!10^{-8}$ & $1.62\!\times\!10^{-8}$\\\bottomrule
\end{tabular}
\end{table}

Factored BF1 has $30.6\times$ lower interaction RMS than delta BF1, despite almost unchanged endpoint RMS. Factored axes BF3 has $61.8\times$ lower interaction RMS than the matched axes-only control. As in the VLM, axis precision alone does not explain the gain. Full delta BF3 is more accurate still. At the ordinary full amplitude $t=1$, Factored axes BF3 has interaction RMS $7.14\times10^{-6}$, compared with $8.68\times10^{-8}$ for full delta BF3. The factor of 61.8 therefore holds against the matched control at $t=1/16$, not against full BF3.

\paragraph{Amplitude scaling.}
The probe uses the fixed three-point window $t\in\{2^{-10},2^{-8},2^{-6}\}$ for two-layer log--log fits. FP64-execution interaction-error slopes are 0.990 for delta BF1, 1.991 for Factored BF1, 1.001 for axes-only BF3, and 2.007 for Factored axes BF3. These descriptive slopes are consistent with the local order analysis. FP32 execution reaches a floor near $3.3\times10^{-8}$, and its factored slopes over the same window are near zero (Fig.~\ref{fig:scaling}). 

\begin{figure}[!htbp]
\centering
\includegraphics[width=\linewidth]{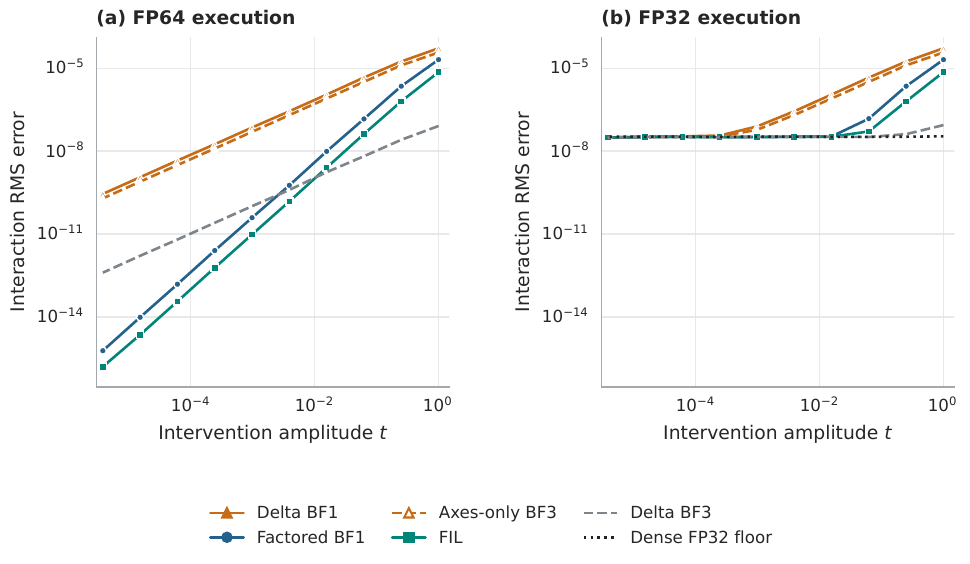}
\caption{CPU amplitude sweep against dense FP64. Shared-axis methods show approximately second-order decay in the fixed FP64 fit window; FP32 curves reach a numerical floor. The network has two RMSNorm--SwiGLU residual blocks and no attention.}\label{fig:scaling}\end{figure}

\paragraph{A failure case.}
The stress condition adds a full-dimensional random mixed component with zero axes to the initial grid. At $t=0.25$ and mixed scale 16, Factored BF1 has interaction RMS $3.81\times10^{-5}$, versus $3.78\times10^{-5}$ for delta BF1. Factored axes BF3 also loses to the matched axes-only control ($3.62\times10^{-5}$ versus $3.58\times10^{-5}$). Sharing therefore does not always win, and the assumption that mixed components are second order matters. This condition injects an independent joint residual into the additive family of Eq.~\eqref{eq:input}, so it tests robustness to large mixed components; it does not show that they occur in our VLM protocol.

\subsection{Configuration and scope}
The probe uses NumPy's PCG64 generator with seed 20261017. Every array is drawn from the standard normal distribution in FP64 using \texttt{Generator.normal}, with the shapes and order in Table~\ref{tab:cpu-draws}. Each weight array is multiplied by its listed scale, and its entries are rounded to BF16-representable values. The first four draws belong to the one-layer diagnostic and fix the generator state at which the two-layer network is built. RMSNorm uses $\epsilon=10^{-5}$ and no learned affine parameters.

\begin{table}[!htbp]\centering\small
\caption{Complete random-draw order for the CPU probe. Rows containing several calls are executed left to right. No random draws occur between these calls. Scales are applied after drawing.}
\label{tab:cpu-draws}
\begin{tabular}{clll}\toprule
Calls & Arrays & Shapes & Scales\\\midrule
1 & One-layer weight & $(64,32)$ & $1/\sqrt{64}$\\
2--4 & One-layer $x,a,b$ & $(512,64)$ each & $1$\\
5--7 & Layer 1 gate, up, down & $(32,48),(32,48),(48,32)$ & $1/\sqrt{32},1/\sqrt{32},0.3/\sqrt{48}$\\
8--10 & Layer 2 gate, up, down & $(32,48),(32,48),(48,32)$ & $1/\sqrt{32},1/\sqrt{32},0.3/\sqrt{48}$\\
11 & Output head & $(32,4)$ & $0.4/\sqrt{32}$\\
12--14 & Two-layer $x_0,a,b$ & $(32,1,1,32)$ each & $1$\\
15 & Mixed bank & $(32,8,8,32)$ & $1$\\\bottomrule
\end{tabular}\end{table}
The two amplitude axes take the eight values $0,1/7,\ldots,1$. After the last draw, the mixed bank's entries with either grid index zero are set to zero; the bank is then reused at every stress scale.

The two-layer sweep covers ten amplitudes $t=2^{-k}$ for $k=0,2,\ldots,18$, plus five large-mixed conditions at $t=0.25$ with mixed scale $0,0.25,1,4,16$. Each case has five approximate methods in FP64 and the same five plus dense in FP32, totaling $15\times11=165$ rows. The large-mixed bank is random over joint entries and zero on the axes; its scale is independent of $t^2$. Every error is computed against the corresponding dense FP64 output, not a common reference from another amplitude.

BF16 casts first represent values as FP32, then implement round-to-nearest-even by bit manipulation. Consequently, ``FP64 execution'' describes accumulation and nonlinear evaluation, not a native arbitrary-precision-to-BF16 conversion. The experiment is limited to its finite generated values and does not validate all subnormal, overflow, or hardware reduction behavior.

The probe uses a single generated network, so it does not measure variation across networks or pretrained models.

\FloatBarrier
\section{Natural-Image Stimuli and Interventions}
\label{app:natural}
The two-question natural-image protocol uses 64 Microsoft COCO val2014 images~\cite{lin2014}, two object-color questions each, and six baseline precision methods. RGB images are resized to $448\times448$ with bicubic interpolation; masks use nearest-neighbor interpolation. Overlapping pixels are removed from both masks. Each cell of the $14\times14$ visual-token grid is assigned to the mask with greater area, with positive ties assigned to A and empty cells unowned. This geometry is fixed across methods. Resizing changes aspect ratio for non-square inputs.

Each donor blends pixels within its mask halfway toward a palette color. Native and individual donor images are forwarded under the same question. Donor-minus-native differences on owned tokens at decoder-12 output define the internal intervention family. Both questions stay within the image cluster. Numerical evaluation covers all 128 cases, but not all questions and masks have been reviewed semantically; the apple question of case 123 is known to be ambiguous.

\paragraph{Single-MLP FIL inputs.}
The FIL study uses four images from the same inventory, with two category questions per image; a fifth image is used only to recheck the failure case of Appendix~\ref{sec:natural-baseline} (Table~\ref{tab:natural_inputs}). Case index is $2\times\text{image index}+q$, with $q\in\{0,1\}$ selecting the first or second annotated category. The question is exactly ``What color is the \emph{category}? Answer with one word.'' An annotation ID fixes the source mask but does not ensure that this verbal reference is semantically unique. The two donors and their geometric supports remain the same for both questions; only the text changes. Native and each individual donor are forwarded under the corresponding question. This source differs from the synthetic panel question and renderer in Appendix~\ref{app:synthetic}.

\begin{table}[!htbp]\centering\small
\caption{Natural inputs from COCO val2014. Category/annotation pairs identify donor A and B masks and questions $q=0,1$. Image 61 is used only to recheck the failure case of Appendix~\ref{sec:natural-baseline}.}
\label{tab:natural_inputs}
\begin{tabular}{rrlll}\toprule
Index & COCO image ID & A: category / annotation & B: category / annotation & Role\\\midrule
0 & 274083 & car / 140224 & bus / 168652 & Main\\
16 & 64307 & cup / 676954 & bowl / 710873 & Main\\
32 & 302219 & backpack / 1165906 & handbag / 2211551 & Main\\
48 & 59652 & dog / 6832 & potted plant / 23168 & Main\\
61 & 108669 & bowl / 1535648 & apple / 1548458 & Recheck\\\bottomrule
\end{tabular}\end{table}

The fixed natural palette is $[(220,40,40),(40,180,70),(40,80,220),(220,190,40)]$. For image index $k$, donor A blends its owned mask halfway toward palette entry $k\bmod4$, and donor B separately blends its mask toward entry $(k+2)\bmod4$. Thus all four primary images use red $(220,40,40)$ for A and blue $(40,80,220)$ for B; image 61 uses green and yellow. Pixel and token ownership follow the rules above.

\paragraph{Implementation variants.}
History dense runs the exact-history executor without approximate projections. Delta BF1 uses packed non-native candidates and storage reuse. Axes-only BF3 and FIL additionally use fixed device indices and fused reconstruction. Full delta BF3 and plain BF3 use their optimized packed implementations with storage and index reuse. Each of the four BF3-based methods is evaluated with both original and batched retrieval of the same audit checks. History dense and delta BF1 have no paired audit variant in this comparison. Table~\ref{tab:natural-original-audit} lists all ten conditions; FIL/axes-only/plain/delta BF3 comparisons in Section~\ref{sec:natural-input} use the batched variants unless explicitly stated otherwise. The batched variant was not chosen per method by observed speed.

\FloatBarrier
\section{Region-Pair Stimuli and Execution}
\label{app:region-inputs}
The region-pair study uses four COCO val2014 images with four fixed regions each (Table~\ref{tab:region-stimuli}). Region 0 supplies the question category: ``What color is the \emph{category}? Answer with one word.'' The six pairs, in order, are $(0,1),(0,2),(0,3),(1,2),(1,3),(2,3)$. These are development inputs, with no independent semantic-accuracy evaluation.

Unlike the two-question study, resizing preserves aspect ratio: the image is fitted inside $448\times448$, with bicubic interpolation and centered padding of RGB $(127,127,127)$. Masks use nearest-neighbor interpolation. Pixel overlap is assigned once in ascending annotation-ID order. Each cell of the $14\times14$ visual-token grid belongs to the region with greatest mask coverage, with positive ties broken by region index and empty cells unowned. This global ownership remains fixed across all six pairs.

Each donor separately blends its region halfway toward one of the four palette colors listed in Appendix~\ref{app:natural}. For image index $k$ and region $r$, the palette entry is $(k+r)\bmod4$. Native and all four donors are forwarded under the same question, and their supported decoder-12 differences define the six intervention families. All inputs contain 217 tokens. The common frontiers are 89, 97, 96, and 92 for indices 0, 3, 13, and 22, respectively. Before the common frontier, even tokens after a pair's own full-branch frontier use dense MLP computation.

\begin{table}[!htbp]\centering\small
\caption{Region-pair study stimuli from COCO val2014. Regions are listed in order 0--3, as category (annotation ID); region 0 is queried. These indices are local to this study.}
\label{tab:region-stimuli}
\begin{tabular}{rrp{.69\linewidth}}\toprule
Index & COCO ID & Regions 0--3\\\midrule
0 & 529850 & cat (50214); person (231125); suitcase (1182980); refrigerator (1646710)\\[3pt]
3 & 302789 & backpack (1167322); suitcase (1185989); bed (1611695); refrigerator (1646862)\\[3pt]
13 & 315577 & dog (11796); couch (111776); laptop (1099186); book (1147889)\\[3pt]
22 & 221911 & orange (1051333); cup (674996); cake (1088518); knife (1892941)\\[3pt]
\bottomrule\end{tabular}\end{table}

The five arithmetic methods are delta BF1, axes-only BF3, FIL, full delta BF3, and full plain BF3, each with both pair-specific and common precision ranges; history dense is the eleventh method. All 264 method--pair records pass the maximum-error and endpoint-MSE gates, with no arithmetic fallback. The main comparison uses the common-range variants. The pair-specific variants repeat an earlier evaluation on the same images and are not additional independent evidence. For all 24 FIL/control pairs, the execution logs report the same multiply--accumulate count, the same packed projection kernel shapes, and the same number of audit transfers.

Timing covers candidate execution, support checks, history planning, native/KV computation, and synchronization. Source acquisition and common-frontier planning are measured once per image and added once per bundle. Output copying, rescoring, and replay comparisons are outside the timer; preparation of the BF16 weight copies and cold JIT compilation are recorded separately. Seven-repeat FIL bundle ranges span 18--79\,ms across the four images. The measured differences of a few milliseconds therefore describe this run, not a stable speed ordering. Ordinary exact execution averages 10.4255\,s per bundle; its larger cost includes repeated histories already avoided by history dense.

For cross-pair reuse, each image has 29 repeated states: native and seven amplitudes for each of four regions. The input audit counts a group as identical only when all its copies agree. Table~\ref{tab:cross-pair-inputs} refers to normalized inputs before FIL, not to the same-pair method comparison of Eq.~\eqref{eq:matched-axes}. It is based on saved input-audit records, without identifying the upstream operator that first introduces the difference.

\begin{table}[!htbp]\centering\small
\caption{Cross-pair input audit at the common precision range in the region-pair study. Max difference is across repeated copies of the same native or single-region state after normalization.}
\label{tab:cross-pair-inputs}
\begin{tabular}{rrrr}\toprule
Index & COCO ID & Identical groups / 29 & Max input difference\\\midrule
0 & 529850 & 1/29 & $1.43\times10^{-6}$\\
3 & 302789 & 8/29 & $1.19\times10^{-6}$\\
13 & 315577 & 8/29 & $1.79\times10^{-6}$\\
22 & 221911 & 8/29 & $1.10\times10^{-6}$\\
\bottomrule\end{tabular}\end{table}

\FloatBarrier
\section{Color Tokens and Model Configuration}
\label{app:tokens}
The reported Qwen studies use \texttt{Qwen/Qwen3-VL-4B-Instruct} with model and tokenizer revision
\begin{center}\small\texttt{ebb281ec70b05090aa6165b016eac8ec08e71b17}.\end{center}
Loading stored parameters into FP32 does not recover unrounded training-time weights; the reference is FP32 execution of this checkpoint. The run configuration specifies Transformers 5.12.1, PyTorch 2.11.0+cu130, FP32 model computation with eager text/vision attention, disabled TF32, and NVIDIA GH200 120GB. Table~\ref{tab:tokens} lists the answer groups; their spellings were checked against \texttt{tokenizer.json} of the same revision.

\paragraph{Input sequence and token positions.}
The checkpoint's chat template receives one user message containing the image followed by the question, with \texttt{add\_generation\_prompt=True}. There is no system message, tool description, or image-number prefix. The user header is \verb+<|im_start|>user+ followed by a newline, then \verb+<|vision_start|>+, 196 \verb+<|image_pad|>+ tokens, \verb+<|vision_end|>+, and the question without an intervening separator. The message ends with \verb+<|im_end|>+ and a newline; \verb+<|im_start|>assistant+ and a newline form the generation prompt. With zero-based sequence positions, the visual tokens occupy 4--199. A merged visual cell $(r,c)$ therefore maps to
\[
 k(r,c)=4+14r+c,\qquad 0\le r,c<14.
\]
Question length changes the total sequence length, but not this visual offset. All synthetic inputs have 218 tokens; the four region-pair inputs have 217.

For pair $p$, define candidates as equivalent at position $k$ when their FP32 intervention states $X_{ij,0:k}$, including position $k$, are identical under the same question. The full-branch frontier $f_p$ is the first $k$ at which all 64 candidates have distinct prefix histories. This definition uses the formed states, so rounded-away changes do not count as a branch. For disjoint supports whose first changed token distinguishes all eight strengths, it reduces to $f_p=\max(\min S_A,\min S_B)$, where $S_A,S_B$ are the changed-token supports in sequence coordinates. The synthetic panels satisfy this condition: their first supported positions are 33, 40, 117, and 124. Coalescing short segments leaves $f_p$ unchanged. The region-pair comparison uses $F=\max_p f_p$ across the six pairs, as specified in Appendix~\ref{app:region-inputs}.

\begin{table}[htbp]
\centering\small
\caption{Fixed color-token groups. Each entry gives token ID followed by its decoded string. The symbol \texttt{\textvisiblespace} denotes exactly one leading ASCII space; all other characters are literal.}
\label{tab:tokens}
\begin{tabular}{lp{.80\textwidth}}\toprule
Color & Token ID: decoded string\\\midrule
red & 1151: \texttt{red}; 2518: \texttt{\textvisiblespace red}; 3731: \texttt{\textvisiblespace Red}; 6033: \texttt{Red}; 6504: \texttt{RED}\\[4pt]
green & 6176: \texttt{\textvisiblespace green}; 7840: \texttt{\textvisiblespace Green}; 13250: \texttt{green}; 19576: \texttt{Green}; 55954: \texttt{GREEN}\\[4pt]
blue & 6303: \texttt{\textvisiblespace blue}; 8697: \texttt{\textvisiblespace Blue}; 10331: \texttt{Blue}; 12203: \texttt{blue}; 64330: \texttt{BLUE}\\[4pt]
yellow & 13753: \texttt{\textvisiblespace yellow}; 25462: \texttt{\textvisiblespace Yellow}; 27869: \texttt{yellow}; 47699: \texttt{Yellow}; 97029: \texttt{YELLOW}\\\bottomrule
\end{tabular}
\end{table}

For each lowercase color name $c$, the group builder tries the lowercase, initial-capital, uppercase, space-prefixed lowercase, and space-prefixed initial-capital spellings, with special tokens disabled. It retains a spelling only if it encodes as one token and decoding, stripping surrounding whitespace, and lowercasing returns $c$. IDs are deduplicated within each group and sorted; any overlap between groups causes an error. Multi-token spellings are excluded, rather than scored by multiplying conditional probabilities. Leading-space uppercase, trailing-space variants, and other spellings are not added by this rule. The 20 IDs are pairwise disjoint. The reported output is a next-token event, not a multi-token color-answer score.

\FloatBarrier
\section{Synthetic Stimuli, Donors, and Token Supports}
\label{app:synthetic}
\paragraph{Source and scene construction.}
The main families are the validation-source indices 208--223, with scene identifiers $133000000+i$.

The image has $448\times448$ pixels and four labeled panels. Labels A, B, C, D have indices $0,1,2,3$. With layout offset $\ell$, label $k$ occupies box $(k+\ell)\bmod4$, where boxes 0--3 (upper left, upper right, lower left, lower right) have inclusive coordinates
\[
 (42,70,190,196),\ (258,70,406,196),\
 (42,286,190,412),\ (258,286,406,412).
\]
The renderer draws a two-pixel black outline and a 24-pixel default-font label at $(x_0+64,y_0-32)$. For scene identifier $s$, each background pixel at zero-based row $y$ and column $x$ has all three channels equal to $170+15[(\lfloor y/14\rfloor+\lfloor x/14\rfloor+s)\bmod2]$. This replaces only the renderer's RGB $(245,245,245)$ background, leaving panels, outlines, and labels intact.

The class order is red, green, blue, yellow, with palette
\[
 P=\bigl[(205,45,45),(45,180,65),(45,75,210),(215,200,40)\bigr].
\]
For source index $i$, we use NumPy's default PCG64 generator with seed $164191009+i$ (the protocol seed 1191009 plus two fixed offsets, 133000000 and 30000000, plus $i$). Six calls to \texttt{Generator.integers} with the default \texttt{int64} dtype sample uniformly, in this order, the four panel classes from $\{0,1,2,3\}$ with replacement, the first edited label $a$ from $\{0,1,2,3\}$, a class offset $o$ from $\{1,2,3\}$, the $4\times3$ native RGB jitters from $\{-18,\ldots,18\}$, three donor-A RGB jitters from the same range, and the layout $\ell$ from $\{0,1,2,3\}$. Donor A's class is $(c_a+o)\bmod4$. Table~\ref{tab:synthetic-rgb} gives the resulting RGB values.

\paragraph{Questions and donors.}
Each family uses one question about label $a$: ``What color is panel \emph{L}? Answer with one color word.'' Table~\ref{tab:synthetic-scenes} specifies \emph{L}. Donor A changes only panel $a$ to its sampled different class, with the recorded RGB jitter. The second edited label is $b=(a+1)\bmod4$. Donor B is constructed separately from the native image: it changes only panel $b$ to $P[(c_b+1)\bmod4]$, without RGB jitter. Every other panel keeps its native RGB in both donors. Unlike the natural-image donors, these donors do not blend pixels. The same question is used for native and both donor forwards.

\paragraph{Pixel cells and hidden supports.}
The Qwen3-VL processor produces $28\times28$ patches of size $16\times16$, merged into $14\times14$ row-major visual tokens. Each merged cell covers $32\times32$ pixels. A token belongs to a panel's cover if its cell intersects the inclusive box: $32c<x_1+1$, $32(c+1)>x_0$, $32r<y_1+1$, and $32(r+1)>y_0$. Each cover contains 25 cells, including boundary cells, and the two covers are disjoint. These cells are mapped to their positions in the 218-token input. Both donor states are taken after decoder 12, and only the donor-minus-native differences on the corresponding cover are inserted into the initial family. Because attention is global, these supports are geometric and do not isolate object representations. The executor merges segments shorter than 32 tokens with the next, finer partition. For family 208, histories first split at token 33 and become fully distinct at token 40. Coalescing starts separate execution at token 33, but tokens 33--39 keep dense MLP computation; FIL starts at token 40.

\begin{table}[!htbp]
\centering\small
\caption{All 16 synthetic configurations. Scene ID is $133000000+i$ and seed is $164191009+i$. Native class tuples are ordered A--D; R/G/B/Y denote red/green/blue/yellow. ``Donor A'' gives the queried/first edited label and its new class; ``Donor B'' gives the separately edited label and its new class. Frontiers are zero-based; active count is $218-$frontier.}
\label{tab:synthetic-scenes}
\begin{tabular}{rclccrr}\toprule
Index $i$ & Layout & Native classes & Donor A & Donor B & Frontier & Active\\\midrule
208 & 3 & BBGY & B$\to$G & C$\to$B & 40 & 178\\
209 & 1 & GGBB & D$\to$G & A$\to$B & 40 & 178\\
210 & 1 & GGBG & D$\to$R & A$\to$B & 40 & 178\\
211 & 3 & YRYG & A$\to$R & B$\to$G & 124 & 94\\
212 & 1 & GBBB & C$\to$Y & D$\to$Y & 124 & 94\\
213 & 0 & GRRY & A$\to$B & B$\to$G & 40 & 178\\
214 & 3 & YRYG & B$\to$G & C$\to$R & 40 & 178\\
215 & 1 & BYBB & B$\to$R & C$\to$Y & 124 & 94\\
216 & 0 & GRBG & A$\to$Y & B$\to$G & 40 & 178\\
217 & 0 & GYYG & D$\to$Y & A$\to$B & 124 & 94\\
218 & 3 & GGBR & B$\to$R & C$\to$Y & 40 & 178\\
219 & 3 & YRBY & A$\to$B & B$\to$G & 124 & 94\\
220 & 1 & BBBY & C$\to$Y & D$\to$R & 124 & 94\\
221 & 2 & BGBB & A$\to$Y & B$\to$B & 124 & 94\\
222 & 3 & RRYY & B$\to$B & C$\to$R & 40 & 178\\
223 & 1 & BGYG & A$\to$Y & B$\to$B & 117 & 101\\
\bottomrule\end{tabular}\end{table}

\begin{table}[!htbp]
\centering\footnotesize
\caption{Recorded RGB fills, ordered A--D, and replacement RGB for donor A. Donor B uses the exact palette RGB of the class listed in Table~\ref{tab:synthetic-scenes}. Values are unsigned image-channel integers.}
\label{tab:synthetic-rgb}
\begin{tabular}{rlllll}\toprule
Index & Native A & Native B & Native C & Native D & Donor A replacement\\\midrule
208 & (59,79,202) & (41,80,193) & (62,177,81) & (204,186,28) & (28,187,64)\\
209 & (30,188,53) & (63,195,78) & (50,62,197) & (54,65,225) & (46,192,65)\\
210 & (38,194,61) & (30,191,56) & (41,72,214) & (52,188,71) & (193,57,48)\\
211 & (223,193,35) & (187,33,43) & (224,204,56) & (40,178,76) & (211,50,61)\\
212 & (54,182,75) & (57,66,203) & (50,60,204) & (30,91,202) & (218,213,27)\\
213 & (63,175,80) & (202,60,28) & (194,34,51) & (207,195,38) & (57,75,223)\\
214 & (198,187,22) & (201,56,57) & (202,209,34) & (34,174,73) & (52,166,63)\\
215 & (55,61,192) & (216,210,42) & (32,62,201) & (33,93,227) & (219,46,43)\\
216 & (33,170,69) & (208,53,32) & (53,76,203) & (32,172,61) & (220,211,35)\\
217 & (48,165,55) & (211,208,36) & (223,195,58) & (40,180,73) & (230,212,39)\\
218 & (46,171,51) & (56,170,80) & (53,61,208) & (209,41,32) & (193,49,57)\\
219 & (210,185,52) & (217,50,49) & (52,92,201) & (199,200,40) & (40,77,200)\\
220 & (46,65,193) & (33,82,199) & (60,72,222) & (199,186,48) & (219,213,30)\\
221 & (58,58,211) & (59,194,55) & (48,82,213) & (38,90,221) & (209,190,34)\\
222 & (209,54,48) & (223,28,39) & (207,195,25) & (201,209,27) & (52,68,203)\\
223 & (48,77,211) & (47,197,76) & (211,190,51) & (28,191,66) & (230,202,41)\\
\bottomrule\end{tabular}\end{table}

\begin{figure}[!htbp]\centering
\includegraphics[width=\linewidth]{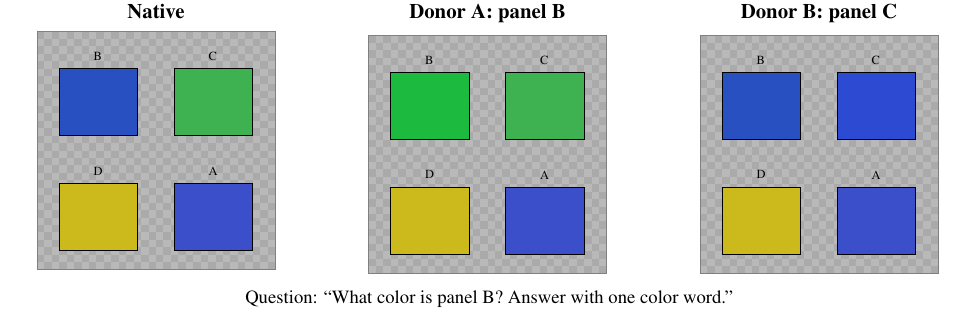}
\caption{Synthetic family 208: native and separate donor images, drawn from the recorded layout and RGB values. Typography is schematic. Donor A changes panel B, donor B changes panel C, and the question asks about B. The intervention grid combines the supported hidden differences.}
\label{fig:synthetic}
\end{figure}

\end{document}